\documentclass[11pt]{article}
\usepackage[utf8]{inputenc} 
\usepackage[T1]{fontenc}    
\usepackage{hyperref}       
\usepackage{url}            
\usepackage{booktabs}       
\usepackage{amsfonts}       
\usepackage{nicefrac}       
\usepackage{microtype}      
\usepackage{xcolor}         

\usepackage[normalem]{ulem}
\usepackage{amsthm}
\usepackage{amsmath}
\usepackage{amssymb}
\usepackage{adjustbox}
\usepackage{algorithm}
\usepackage{algpseudocode}
\usepackage{enumitem}
\usepackage{wrapfig}

\newtheorem{theorem}{Theorem}

\newtheorem{remark}{Remark}
\newtheorem{assumption}{Assumption}
\usepackage{natbib}
\usepackage{comment}
\usepackage{rotating}
\usepackage{subcaption}
\usepackage{pdflscape}
\usepackage{caption}
\usepackage{adjustbox}

\usepackage[margin = 1in]{geometry}

\hypersetup{
    colorlinks=true,
    citecolor={blue!50!black},
    urlcolor={cyan!40!black},
    linkcolor={red!60!black}
}

\usepackage{fontawesome5} 

\newcommand{\mask}{\mathsf{m}}

\newcommand{\bx}{\mathbf{x}}
\newcommand{\calX}{\mathcal{X}}

\newcommand{\algname}{MaskDiff-AD}
\newcommand{\rec}{\mathrm{rec}}
\newcommand{\KL}{\mathrm{KL}}

\newcommand*{\ie}{\emph{i.e.}{}}

\algrenewcommand\algorithmicrequire{\textbf{Input:}}
\algrenewcommand\algorithmicensure{\textbf{Output:}}

\newcommand{\yuchen}[1]{\color{cyan}{Yuchen: #1 }\color{black}}

\title{\bfseries\LARGE Masked Diffusion Modeling for Anomaly Detection}
\date{}

\begin{document}

\maketitle

\vspace{-5em}

\begin{center}

{\large
Lixing Zhang$^{1}$ \qquad
Yuchen Liang$^{2}$ \qquad
Liyan Xie$^{1}$
}

{\large
$^{1}$University of Minnesota
\qquad
$^{2}$Ohio State University
}

\vspace{0.8em}

\faGithub\;Code: \href{https://github.com/lxzhang1/MaskDiff-AD}
{https://github.com/lxzhang1/MaskDiff-AD}

\begingroup
\renewcommand\thefootnote{}
\footnotetext{\it \hspace*{-1.8em}Main contact: liyanxie@umn.edu
}
\endgroup

\vspace{0.5em}

\end{center}


\begin{abstract}
Anomaly detection aims to identify samples that deviate from the nominal data distribution and is central to many safety-critical applications. However, developing effective anomaly detection methods for categorical, mixed-type, and discrete sequence data remains challenging and relatively underexplored. Masked diffusion models provide a natural way to model such data by learning to recover masked values from the remaining visible context. In this paper, we propose \emph{Masked Diffusion for Anomaly Detection} (MaskDiff-AD), a forward-only method based on masked diffusion models trained only on nominal data. Given a test sample, MaskDiff-AD constructs anomaly scores from the difficulty of reconstructing randomly masked coordinates, yielding a content-sensitive score that operates directly on discrete state spaces while avoiding reverse-time sampling. We also develop a non-parametric variant of MaskDiff-AD and provide theoretical guarantees by characterizing Type-I and Type-II errors under a fixed detection threshold. Experiments on fourteen categorical and mixed-type tabular datasets from ADBench and UADAD, as well as four text anomaly detection datasets from NLP-ADBench, show that MaskDiff-AD achieves competitive performance against classical, diffusion-based, and recent tabular/text anomaly detection baselines. Notably, MaskDiff-AD achieves the best overall average rank, outperforming all twelve tabular baseline methods.
\end{abstract}

\section{Introduction}
\label{sec:discrete_dte}

Anomaly detection aims to identify observations that deviate significantly from the majority of the data (i.e., the data-generating distribution) \cite{chandola2009anomaly,ruff2021unifying,pang2021deep}. As a core problem in machine learning and statistics, it has been studied for decades and plays an important role in a wide range of applications, including healthcare \cite{salem2013sensor,pachauri2015anomaly}, finance \cite{ahmed2016survey}, cybersecurity \cite{ahmed2016survey}, manufacturing \cite{susto2017anomaly}, particle physics \cite{fraser2022challenges}, and geospatial analysis \cite{yairi2006telemetry}. 
Despite this long history, the increasing prevalence of large-scale datasets with discrete or mixed data types has exposed important limitations of many classical methods \cite{taha2019anomaly}. 
Many modern anomaly detection problems involve categorical, discrete, or
mixed-type data, including tabular records with categorical attributes,
transaction and claims data with heterogeneous feature types, and tokenized text
sequences \cite{han2022adbench,li2025nlp}. These data types are not naturally suited to methods built around Euclidean geometry or continuous densities.  
These challenges motivate the development of anomaly detection methods that are well-suited to the structural characteristics of discrete-type data, while being scalable, interpretable, and sufficiently expressive.

Recently, diffusion models have provided a flexible framework for learning complex data distributions \cite{sohl2015deep,ho2020denoising}. Among them, discrete diffusion models, especially masked diffusion models, have further extended this capability to discrete domains such as text, music, and molecular data
\cite{sahoo2024simple,austin2021structured}. Notably, on certain tasks, masked diffusion models have demonstrated superior performance compared to autoregressive approaches \cite{zheng2025fhs,prabhudesai2026diffusion}. However, as with other generative models, the majority of research on diffusion models has focused on improving sampling efficiency and enabling controllable generation (e.g., \cite{zheng2025fhs,cardei2025constrained,uehara2024survey-finetuning}). In contrast, relatively few studies have explored their use, particularly that of masked diffusion models, for anomaly detection.

In the literature, one notable work has explored the use of (continuous) diffusion models for anomaly detection: the Diffusion Time Estimation (DTE) framework \cite{livernoche2024dte}. This approach estimates the posterior distribution over the diffusion time and uses the inferred corruption level as an anomaly score, thereby avoiding the need for expensive reverse-time sampling.
However, a key component of this framework relies on the Gaussian corruption structure, with a possible discrete analogue in uniform discrete diffusion models. As we show in Section~\ref{sec:method}, this approach does not readily extend to masked diffusion models. In particular, due to the absorbing masking mechanism, the DTE-style posterior becomes ill-posed: it either collapses to identical values across all clean categorical samples or depends solely on the masking pattern of the corrupted observation, rather than the underlying distribution of the test data.


Masked diffusion models (MDM) have recently gained increasing attention for modeling discrete data \cite{sahoo2024simple}. They learn to reconstruct clean
categorical values from partially masked observations, such as learning to predict masked tokens from the remaining context. Motivated by the masking procedure, we propose \emph{Masked Diffusion for Anomaly Detection} (\algname{}). 
The key idea is to employ a \emph{mask-corruption level} tailored to the masked discrete diffusion process, and to use the corresponding \emph{single-step reconstruction difficulty} as the anomaly score. In other words, instead of asking at which corruption time a test sample becomes indistinguishable from noisy normal data, \algname{} asks how difficult it is to reconstruct masked coordinates from the remaining visible context under a model trained only on normal data. A normal sample should have coordinates that are mutually consistent and therefore easy to predict from one another, whereas
an anomalous sample should induce larger masked reconstruction surprisal. This yields a content-sensitive anomaly score without reverse-time sampling or
iterative imputation.

The contributions of our work are summarized as follows: 
\begin{itemize}
    \item We formulate anomaly detection on discrete state spaces using masked diffusion models and show why a direct DTE-style time posterior degenerates in this setting. 
    
    \item We introduce \algname{}, a masked reconstruction-based anomaly detection method
    that is trained \textit{only on nominal samples}, computationally efficient, and does not require reverse-time sampling. We also develop both non-parametric and parametric implementations of
    \algname{}.

    \item We provide theoretical analysis for the proposed reconstruction score, including guarantees on Type-I and Type-II errors.

    \item We demonstrate the effectiveness of our \algname{} with extensive experiments on categorical tabular, mixed-type tabular, and text anomaly detection benchmarks.
\end{itemize}

\vspace{-0.1in} 
\paragraph{Related Work.} 
Our work is connected to several lines of research. Classical anomaly detection
methods include distance-based approaches such as \(k\)-nearest neighbors~\cite{knn},
tree-based methods such as Isolation Forest~\cite{iforest}, and distributional
or copula-based scores such as ECOD~\cite{9737003} and COPOD~\cite{copod}.
These methods are often efficient and broadly applicable, but their performance
can depend strongly on the chosen representation or distance, especially for
categorical and mixed-type data~\cite{taha2019anomaly,chawda2022uadad}. 
Recent tabular anomaly detection methods seek
to capture richer feature dependencies, including masked cell modeling
(MCM)~\cite{yin2024mcm}, non-parametric transformers for anomaly detection
(NPT-AD)~\cite{pmlr-v235-thimonier24a}, internal contrastive learning
(ICL)~\cite{shenkar2022anomaly}, CompreX~\cite{comprex}, and decomposed
representation learning~\cite{ye2025drl}. Our method shares the goal of
modeling feature dependencies, but does so through the forward masking and
conditional reconstruction structure of absorbing discrete diffusion. 

Generative models for anomaly detection have been extensively studied, especially
for image data. Representative examples include autoencoder and variational-autoencoder based detectors~\cite{sakurada2014anomaly,an2015variational}, hybrid autoencoding-density models such as DAGMM~\cite{zong2018deep},
GAN-based detectors such as AnoGAN and GANomaly~\cite{schlegl2017anogan,akcay2018ganomaly}, and flow-based visual
detectors such as DifferNet~\cite{rudolph2021same}. 

Diffusion models have also recently been explored for anomaly detection through
reconstruction-based, density-based, and hybrid scoring
mechanisms~\cite{liu2025survey}. Representative diffusion-based anomaly
detection methods include medical and visual reconstruction methods such as
AnoDDPM~\cite{wyatt2022anoddpm}, DDAD~\cite{mousakhan2024anomaly}, and
DiffusionAD~\cite{zhang2025diffusionad}. These methods are primarily developed for continuous data domains, especially for (medical) image data, where anomaly
scores are often derived from denoising error, likelihood surrogates, or
reconstruction discrepancy after reverse diffusion.

Diffusion time estimation (DTE) \cite{livernoche2024dte} is a recent method for anomaly detection and is efficient because it replaces expensive iterative reverse-chain sampling with a single posterior evaluation, yet retains strong empirical performance on benchmarks. However, its construction relies on the geometry of Gaussian corruption and does not remain discriminative under absorbing masked corruption. Our work instead adapts the diffusion perspective to discrete data by replacing time-posterior estimation with masked reconstruction difficulty.
This keeps the computational advantage of forward-only scoring while producing
a content-sensitive anomaly signal on the original discrete state space.

\section{Preliminaries}\label{sec:pre}

We consider anomaly detection over a general discrete state space $\mathcal{X} = \mathcal{X}_1 \times \cdots \times \mathcal{X}_d$ and denote each data point as a vector $\bx=(x_1,\ldots,x_d)\in\mathcal{X}$, where each coordinate $x_j \in \mathcal{X}_j$ takes values in a finite (and potentially different) state space.  The goal is to distinguish samples generated from a nominal data-generating distribution from anomalous samples that deviate from this distribution. We denote the nominal distribution by $q_0$ and assume that the training data consist {\it only} of normal samples $\mathcal{D}_{\mathrm{tr}}=\{\mathbf{x}_i\}_{i=1}^n$ that are independently and identically sampled from $q_0$.  At test time, a sample may either follow the same nominal distribution or arise from an unknown anomalous distribution $q'$ that differs significantly from $q_0$. The anomalous distribution is not specified during training and no labeled anomalous examples are assumed to be available for training.


Given a test sample $\bx\in\mathcal X$, an anomaly score is assigned such that a larger score indicates that $\bx$ is less likely to have been generated from the normal regime.  Formally, we define an anomaly detector as a scoring function $S:\calX\to \mathbb{R}$ and declare a test sample $\bx$ to be anomalous if $S(\bx)>\gamma$, where $\gamma$ is the detection threshold chosen by controlling the false alarm rate. One example is to choose $\gamma$ by requiring $\mathbb{P}_{\bx\sim q_0}\big(S(\bx)>\gamma\big)\leq \alpha$ for a prescribed significance level $\alpha$. Since the anomalous
distribution is unknown, the main challenge is to construct such a score function that is sensitive to distributional deviations while relying only on normal training data.

\subsection{Masked Diffusion Models}
\label{sec:prelim_mdm}

Masked diffusion models (MDM) provide a generative framework for modeling discrete data by learning to reconstruct clean samples from randomly masked
observations at each generation step \cite{sahoo2024simple}. Trained on nominal data, MDMs can be viewed as learning the underlying distributional characteristics of the nominal distribution $q_0$. 
As the name suggests, MDMs introduce
a special mask state $\mask$ that does not belong to any $\mathcal{X}_j$, resulting in the augmented space:
\[
    \widetilde{\mathcal{X}}
    := \widetilde{\mathcal X}_1\times\cdots\times\widetilde{\mathcal X}_d, \quad \text{ with } \widetilde{\mathcal X}_j := \mathcal X_j \cup \{\mask\}.
\] 
MDMs include a forward and a reverse process.
Given a clean nominal sample $\bx_0\sim q_0$, the forward process constructs a family of intermediate random variables
$\{\bx_t\}_{t \in [0,1]}$ that interpolate between clean data $\bx_0$ and the fully masked state $\bx_1=\{\mask,\ldots,\mask\}$. This is implemented as a \emph{masking (a.k.a., absorbing-state) process} with increasing probability to mask over time. Specifically, let $\alpha_t \in [0,1]$ be a monotone decreasing function with $\alpha_0 \approx 1$ and $\alpha_1 = 0$. With the standard setup in \cite{sahoo2024simple,austin2021structured,shi2024simplified}, one can show that the conditional forward distribution is
\begin{align}\label{eq:forward-mdm}
    q(\bx_t \mid \bx_0)
    =
    \prod_{j=1}^d
    \mathrm{Cat}\big(x_{t,j};\, \alpha_t \delta_{x_{0,j}} + (1-\alpha_t)\delta_{\mathsf{m}}\big),
\end{align}
where $\delta_{x_{0,j}}$ is a delta function centered at $x_{0,j}$.
Moreover, once a coordinate becomes masked, it remains masked for all later times, reflecting the absorbing property.


Then, one way to define the reverse process is to model $q_{s|t,0}(\bx_s \mid \bx_t, \mu_\theta(\bx_{t},t))$ for $s < t$. Here $\mu_\theta^j(\bx_{t},t) := p_\theta^j(\cdot\mid \bx_t)$ is a distribution over $\mathcal{X}_j$ and is
intended to approximate the conditional law of the clean coordinate $x_{0,j}$ given the masked view $\bx_t$.
To train $p_\theta^j$, one can employ the standard masked diffusion training objective, which is a weighted masked reconstruction loss, as follows:
\begin{equation}\label{eq:lossMDM}
    \mathcal{L}_{\mathrm{MDM}}(\theta)
    =
    \int_{0}^1 \mathbb{E}_{\bx_0\sim q_0,\; \bx_t\sim q_{t \mid 0}(\cdot\mid \bx_0)}
    \bigg(
        \sum_{j:x_{t,j}=\mathsf{m}}
        w(t)\log p_\theta^j(x_{0,j}\mid \bx_t)
    \bigg) \mathrm{d} t,
\end{equation}
where $w(t) := \alpha_t'/(1-\alpha_t)$ is a function of $\alpha_t$. 
Essentially, the objective learns to reconstruct the clean data from the partially masked state. 
%
After training on nominal data, the learned
$p_\theta^j(\cdot\mid \bx_t)$ captures the data distribution of the normal regime. Intuitively, a test sample should be considered normal if
its coordinates are predictable from the remaining coordinates under the learned
normal-data conditionals. This observation motivates using masked
reconstruction or conditional likelihood quantities derived from the masked diffusion model as anomaly scores in the next Section.




\begin{figure}[t]
    \centering    \includegraphics[width=0.99\linewidth]{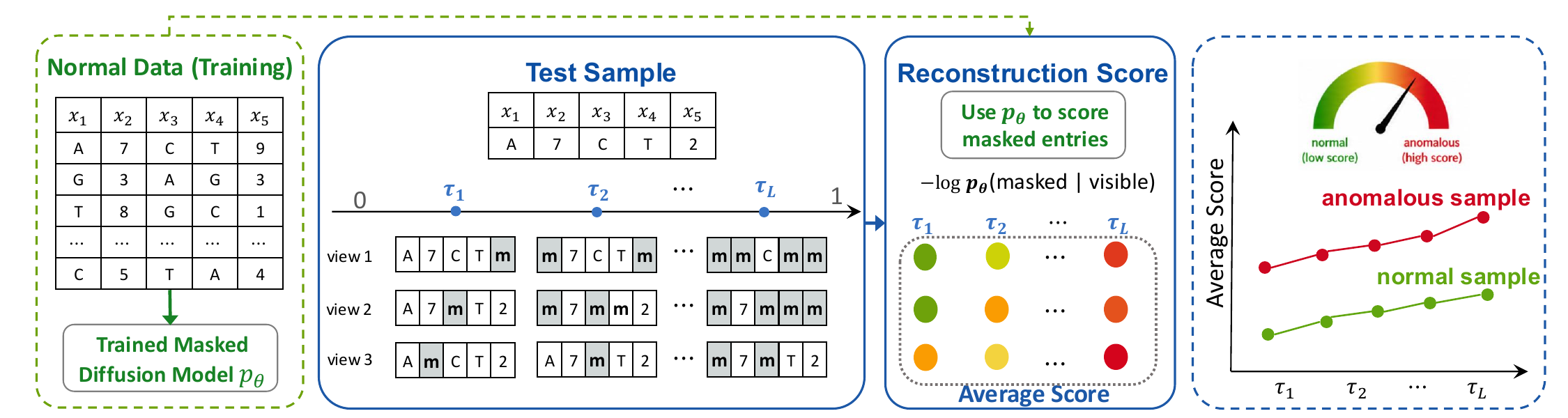}
    \caption{Overview of \algname{}. We first generate masked probe views of a test sample at multiple levels \(\tau_1,\ldots,\tau_L\), and score each view by the reconstruction surprisal of the masked entries using a trained MDM \(p_\theta\) on normal-only data. The final anomaly score averages these scores across views and mask levels; larger score values indicate stronger anomaly evidence.}
    \label{fig:framework}
\end{figure}

\section{Methodology}\label{sec:method}





\subsection{\algname{}: Masked Diffusion for Anomaly Detection}

Our proposed method consists of two steps. \emph{(i) Probe masking:} given a test sample \(\bx\), we generate independent
masked views \(\widetilde\bx^{(\ell,k)}\sim q_{\tau_\ell}(\cdot\mid \bx)\) at
multiple masking levels \(\tau_1,\ldots,\tau_L\). 
\emph{(ii) Reconstruction scoring:} for each masked view, we use the trained
masked diffusion model \(p_\theta\) to compute the reconstruction surprisal of
the masked values, and then average these scores across all views and masking
levels. The framework is illustrated in Figure~\ref{fig:framework}, and we detail these two steps below.

\vspace{-0.1in}
\paragraph{Probe masking.} \label{subsec:probe_schedule}

Let $0 < \tau_1 < \tau_2 < \cdots < \tau_L < 1$
be a pre-specified probe grid. At each probe level \(\tau_\ell\), each coordinate of the test sample $\bx\in\mathcal X$ is replaced by the mask $\mask$ independently with the mask probability $1-\alpha_{\tau_\ell}$ as shown in Eq.~\eqref{eq:forward-mdm}. This induces the following distribution of the masked view $\tilde \bx$ under level $\tau_\ell$:
\begin{equation}\label{eq:probe}
q_{\tau_\ell}(\widetilde \bx\mid \bx)
    :=
    \prod_{j=1}^d
    \left[
        (1-\alpha_{\tau_\ell}) \mathbf 1\{\widetilde x_j=\mask\}
        +
        \alpha_{\tau_\ell}\mathbf 1\{\widetilde x_j=x_j\}
    \right],    
\end{equation}
which can be viewed as the special case of Eq.~\eqref{eq:forward-mdm} under time $t=\tau_\ell$. 
For a masked view \(\widetilde \bx\), we write $M(\widetilde \bx) := \{j:\widetilde x_j=\mask\}$ and $V(\widetilde \bx) := \{j:\widetilde x_j\neq \mask\}$ for the masked and visible coordinate sets. 

Given a chosen integer $K$, we generate $K$ independent probe-masked views \(\widetilde \bx^{(\ell,1)},\ldots,\widetilde \bx^{(\ell,K)}\) for each probe level $\tau_\ell$:
\begin{equation}\label{eq:prob-sample}
   \widetilde \bx^{(\ell,1)},\ldots,\widetilde \bx^{(\ell,K)} \overset{iid}{\sim} q_{\tau_\ell}(\cdot\mid \bx), \text{ for } \ell=1,2,\ldots,L.
\end{equation}
Larger values of \(\tau_\ell\) on average remove more information from the clean sample and therefore produce more heavily corrupted probe views. It is worthwhile noting that the probe levels
\(\tau_1,\ldots,\tau_L\) and the number of views \(K\) are hyperparameters that can be tuned for detection performance. In this work, we use a fixed, uniform
probe grid in the main experiments for simplicity. We also perform a sensitivity analysis on probe levels in our numerical experiments.


\vspace{-0.1in}
\paragraph{Reconstruction scoring.}

We propose to compute the anomaly score by evaluating \emph{how difficult it is to recover the masked coordinates} of a test point under the predictive model trained on nominal data.
Formally, denote the masked diffusion model trained from Eq.~\eqref{eq:lossMDM} as $p_{\hat\theta}^j(\cdot \mid \widetilde \bx)$, \ie, the estimated conditional
distribution of the $j$-th coordinate of the clean data. Then we define the single-probe reconstruction score as
\begin{equation}
    \label{eq:rec_score_general}
    s_\rec(\bx;\widetilde \bx)
    :=
    -\frac{1}{|M(\widetilde \bx)| \vee 1}\displaystyle\sum_{j \in M(\widetilde \bx)}
        \log p_{\hat\theta}^j(x_j \mid \widetilde \bx),
\end{equation}
\ie, the average surprisal of the true values at masked positions.
Intuitively, a normal sample should have a lower reconstruction score, because
its coordinates are consistent with the dependency structure learned from normal
data and are therefore easier to predict from the visible context. In contrast, an anomalous sample is more likely to violate this dependency structure, making
its masked coordinates harder to reconstruct and leading to a higher score, as visualized in Figure~\ref{fig:framework} and confirmed by the toy synthetic example in Figure~\ref{fig:synthetic_example}.


To reduce sensitivity to a particular probe level or random mask, and to improve robustness across various anomalous scenarios, we average the reconstruction score in Eq.~\eqref{eq:rec_score_general}
over \(L\) probe levels and \(K\) independent
masked views per level. Specifically, let
\(\{\widetilde \bx^{(\ell,k)}\}_{\ell\in[L],\,k\in[K]}\) be sampled as in
Eq.~\eqref{eq:prob-sample}, the final aggregated anomaly score is 
\begin{equation}\label{eq:final_score}
 S_{\rec}(\bx)
    :=
    \frac{1}{LK}
    \sum_{\ell=1}^L \sum_{k=1}^K
    s_\rec\!\left(\bx;\, \widetilde \bx^{(\ell,k)}\right).    
\end{equation}
Anomaly detection is then performed by comparing \(S_\rec(\bx)\) with a pre-specified threshold, where test data with larger scores are classified as anomalous. We refer to the resulting method as \emph{Masked Diffusion for Anomaly Detection} (\algname{}). The full algorithm is summarized in Algorithm~\ref{alg:rec_scoring}.

\begin{remark}[Comparison with DTE \cite{livernoche2024dte}]
Our work is inspired by the diffusion time estimation (DTE) for anomaly detection \cite{livernoche2024dte}, which computes the posterior of the diffusion time given the test sample and uses the posterior mean as the anomaly score. We would like to emphasize that DTE does not transfer naturally to MDM here. A natural attempt to extend DTE to MDM is to retain its corruption-time posterior, $q(t\mid\bx)\propto q(t)q_t(\bx)$. In masked diffusion, however, $q_t$ is defined on the extended space $\widetilde{\mathcal X}$ that includes the mask token $\mask$. If a clean sample
$\bx\in\mathcal X$ is viewed as an all-visible element of this space, then $q_t(\bx)=\alpha_t^d q_0(\bx)$, so the resulting posterior $q(t\mid\bx)\propto q(t)\alpha_t^d$ is identical for every clean test sample and therefore contains no discriminative signal. This motivated us to propose the reconstruction-based anomaly score instead under MDM.   
\end{remark}
%

\begin{algorithm}[t]
\caption{Test-Time \algname{}}
\label{alg:rec_scoring}
\begin{algorithmic}[1]
\Require Test sample \(\bx\); pre-trained masked diffusion model \(p_{\hat\theta}^j(\cdot\mid \bx_t) \) for $j=1,\ldots,d$; probe grids \(\{\tau_\ell\}_{\ell=1}^L\); number of views per mask level \(K\); detection threshold \(\gamma\).
\Ensure Anomaly label $\hat y(\bx)\in\{0,1\}$, where 1 indicates anomalous and 0 indicates normal.
\For{each \(\ell=1,\ldots,L\) and \(k=1,\ldots,K\)}
        \State Sample the corresponding probe-masked view
        $\widetilde \bx^{(k,\ell)}\sim q_{\tau_\ell}(\cdot\mid \bx)$ according to Eq.~\eqref{eq:probe}.
        \State Compute per-view reconstruction difficulty
    $s_\rec^{(k,\ell)}$ according to Eq.~\eqref{eq:rec_score_general}.
\EndFor
\State Compute the aggregated anomaly score $S_\rec(\bx)$ according to Eq.~\eqref{eq:final_score}.
\State \Return Anomaly label $\hat y(\bx)=\mathbf 1\{S_{\rec}(\bx)>\gamma\}$.
\end{algorithmic}
\end{algorithm}

\subsection{Some Variants of the Proposed Method} \label{sec:nonpara}

In this subsection, we discuss two practical variants of the proposed \algname{} framework. First, for small-scale categorical datasets, especially
when many coordinates have binary or low-cardinality state spaces, we introduce a non-parametric version of the reconstruction score that estimates masked-coordinate conditional probabilities directly from normal training samples. This avoids the training of parametric MDMs. Second, when a pre-trained MDM is unavailable and one does not need to re-train a generic MDM, we provide a slightly different alternative loss to Eq.~\eqref{eq:lossMDM}, for training the predictive models $p_\theta^j(\cdot \mid \widetilde \bx)$. This objective is more aligned with our anomaly score and can be useful especially for single-probe scoring. 

\vspace{-0.1in}
\paragraph{Non-parametric reconstruction score.} 
For a probe-masked view \(\widetilde \bx\), instead of training a predictive parametric model \(\{ p_{\hat\theta}^j(\cdot\mid \bx_t), j=1,\ldots,d\}\) to compute the anomaly score, we can also adopt the kernel-smoothed empirical conditional distribution as follows. Recall that \(M(\widetilde \bx)\) and \(V(\widetilde \bx)\) denote its masked and visible coordinate sets, respectively. Specifically, for each normal training sample \(\bx^{(n)}\in\mathcal D_{\rm tr}\), define the visible-coordinate Hamming distance between $\bx^{(n)}$ and $\widetilde \bx$ as
\[
    d_{\rm vis}(\widetilde \bx,\bx^{(n)})
    :=
    \sum_{r\in V(\widetilde \bx)}
    \mathbf 1\{\widetilde x_r\neq x^{(n)}_r\}.
\]
Then we estimate the reconstruction probability by
\begin{equation}\label{eq:nonpara}
\widehat p_{\mathrm{NP}}^j(x_j = a \mid \widetilde \bx)
:=
\frac{
\sum_{n=1}^N \mathbf 1\{x^{(n)}_j=a\}\,K_\lambda(\widetilde \bx,\bx^{(n)})
}{
\sum_{n=1}^N K_\lambda(\widetilde \bx,\bx^{(n)})
}, \quad \forall j\in M(\widetilde \bx), a\in\calX_j.   
\end{equation}
Here $K_\lambda(\widetilde \bx,\bx^{(n)})=\exp\bigl(-\lambda d_{\rm vis}(\widetilde \bx,\bx^{(n)})\bigr)$ is the kernel weight and \(\lambda>0\) is a pre-specified bandwidth parameter. The non-parametric anomaly score can be computed similarly as in Algorithm~\ref{alg:rec_scoring} by replacing the parametric predictive model $p_{\hat\theta}^j(\cdot\mid \bx_t)$ with the non-parametric model in Eq.~\eqref{eq:nonpara}. The algorithm is given in Algorithm~\ref{alg:rec_scoring_nonpara} in the Appendix for completeness. 

To illustrate the non-parametric \algname{} score and compare it with the parametric score in Eq.~\eqref{eq:final_score}, we evaluate both on a synthetic setting with varying anomaly levels in Figure~\ref{fig:synthetic_example}, where
\(r=0\) (bottom rows in each heatmap) corresponds to fully anomalous samples and \(r=1\) (top rows in each heatmap)  to fully nominal samples. As samples move away from the normal distribution, the reconstruction
posterior assigned to the true masked value decreases and the corresponding
surprisal increases. This implies that both parametric and non-parametric scores can identify anomalous samples; the non-parametric score is comparable at smaller
mask rates, while the parametric score becomes more effective at larger mask
rates. For small datasets with simple discrete structure, however, the
non-parametric version may be preferable because it requires no model training.

\begin{figure}[t!]
    \centering
    \begin{tabular}{cc}
        \includegraphics[width=0.45\linewidth]{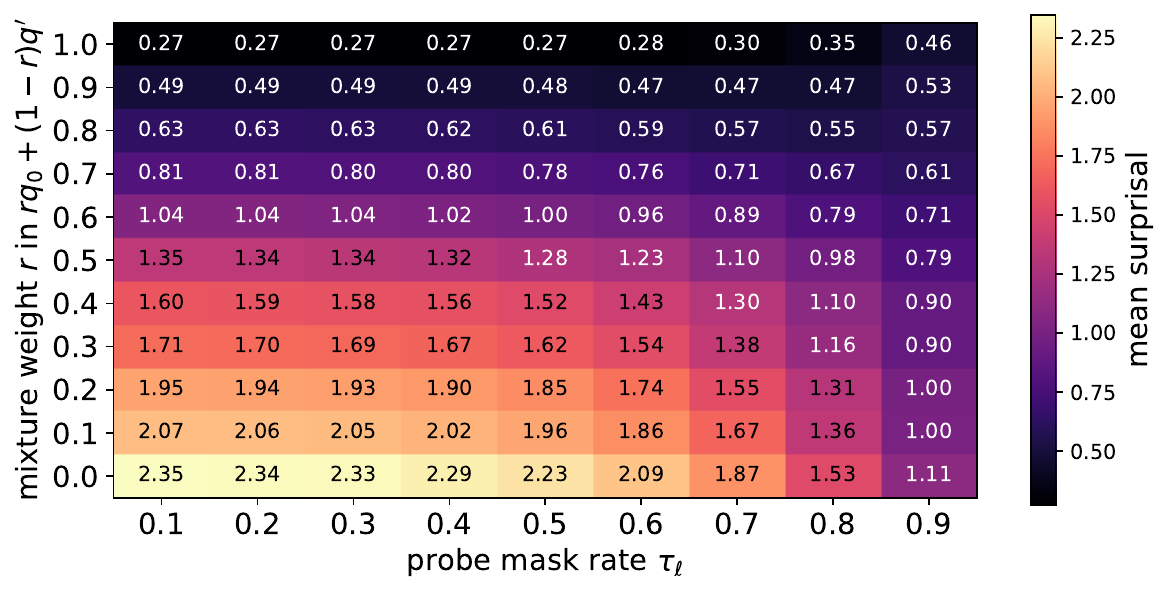} &
        \includegraphics[width=0.45\linewidth]{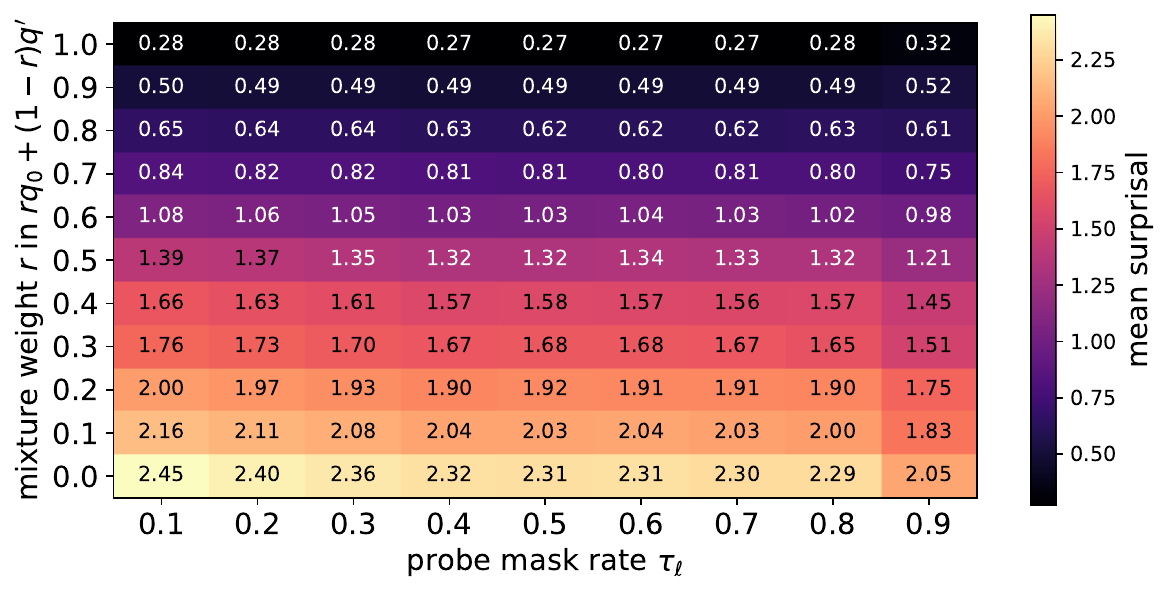} \\
        \vspace{-0.02in}
        (a) Non-parametric score \(\mathbb{E}[-\log \hat p^j_{\mathrm{NP}}(x_j \mid \widetilde{\bx})]\)  & (b) Parametric score \(\mathbb{E}[-\log p^j_{\hat\theta}(x_j \mid \widetilde{\bx})]\) \\
    \end{tabular}
    \vspace{-0.07in}
    \caption{Synthetic heatmaps of the expectation of non-parametric and parametric reconstruction anomaly score under the mixture \(\mu_r = r q_0 + (1-r) q'\). Here, \(q_0\) denotes the normal distribution over length-40 binary vectors, generated by drawing \(S \sim \mathrm{Bernoulli}(0.5)\) and making each coordinate a noisy copy of \(S\) with an 8\% bit-flip probability. The anomaly distribution \(q'\) follows the same construction except that the target coordinate \(j\) is flipped relative to \(S\), preserving its marginal while changing its conditional relationship with the visible context. For each \(\bx \sim \mu_r\), we mask coordinates at rate \(1-\alpha_{\tau_\ell}\) and then force the target coordinate \(j\) to be masked. }
    \label{fig:synthetic_example}
    \vspace{-0.1in}
\end{figure}


\vspace{-0.1in}
\paragraph{An alternative training objective.}
\label{sec:parametric_discrete_dte}

When a pre-trained MDM is available, \algname{} can be applied directly using
the trained model without additional training. When such a model is unavailable
and the goal is anomaly detection, it is not always necessary to train a generic MDM over all diffusion times. Instead, it suffices to learn the conditional
distributions needed for the reconstruction score at the probe levels used at test time. To this end, we may also choose to train a time-conditioned predictive model \(p_{\theta_{\rm rec}}^j(\cdot\mid \widetilde\bx,\ell)\), which estimates the
distribution of the clean coordinate \(x_j\) given a masked view \(\widetilde\bx\) and the probe-level index \(\ell\). Specifically, for each normal training sample \(\bx^{(n)}\in\mathcal D_{\rm tr}\) and each
probe level \(\tau_\ell\), we draw a masked view
\(\widetilde\bx^{(\ell,n)}\sim q_{\tau_\ell}(\cdot\mid \bx^{(n)})\). We then
fit the predictive model by minimizing the average masked reconstruction loss
\begin{equation}\label{eq:para_rec_loss_view_level}
\mathcal L_{\mathrm{rec}}(\theta_{\mathrm{rec}})
= -\frac{1}{NL}
\sum_{n=1}^N \sum_{\ell=1}^L
\frac{1}{|M(\widetilde \bx^{(\ell,n)})|\vee 1}
\sum_{j\in M(\widetilde \bx^{(\ell,n)})} \log
p_{\theta_{\rec}}^j
\bigl( x_j^{(n)} \mid \widetilde \bx^{(\ell,n)},\ell \bigr).
\end{equation}
This objective directly matches the test-time reconstruction score: the model is
trained to predict only the coordinates that are masked in the probe views, and
the loss is normalized per view to make different masking levels comparable.
Averaging over \(\ell=1,\ldots,L\) trains the model for the full multi-probe
score used in Eq.~\eqref{eq:final_score}; for single-probe scoring, one can set
\(L=1\) and train only at the chosen probe level. In this sense, the objective is
tailored to anomaly scoring rather than to general-purpose masked diffusion
generation.


The training loss in Eq.~\eqref{eq:para_rec_loss_view_level} closely resembles that of MDM in Eq.~\eqref{eq:lossMDM}, with two key distinctions. First, our $\mathcal L_{\mathrm{rec}}$ is normalized by the number of masked coordinates, whereas this normalization is replaced with a time-dependent weight $w(t)=\alpha_t'/(1-\alpha_t)$ in the standard MDM loss. Second, the loss in Eq.~\eqref{eq:para_rec_loss_view_level} is only computed with the fixed $L$ mask levels that will be used in detection. This seemingly minor difference reflects the distinct objectives of the two settings.
For sampling, \cite{sahoo2024simple,shi2024simplified} showed that the MDM loss admits a clear interpretation as the KL divergence between probability paths, and minimizing it ensures effective generation along the diffusion trajectory. In contrast, for anomaly detection, our normalized $\mathcal L_{\mathrm{rec}}$ is directly aligned with the test-time scoring used in \algname{}.
Furthermore, the loss $\mathcal L_{\mathrm{rec}}$ is expected to be trained exclusively on nominal data, while a pre-trained MDM might be trained on a more diverse dataset as in general-purpose language modeling.

\subsection{Performance Guarantees}\label{sec:theory}




In this subsection, we provide some theoretical insights for the Type-I and Type-II errors of the proposed algorithm under a well-specified reconstruction model.
Let $q^j(\cdot\mid \widetilde \bx,\tau_\ell)$ denote the oracle conditional distribution at probe level \(\tau_\ell\), and let the corresponding oracle reconstruction score be
\begin{equation}
\label{def:S_ast}
    S^\ast(\bx)
    :=
    \frac{1}{L}
    \sum_{\ell=1}^L
    \mathbb E_{\widetilde \bx \sim q_{\tau_\ell|0}(\cdot\mid \bx)}
    \left[
    -\frac{
        1
    }{
        |M(\widetilde \bx)|\vee 1
    }
    \sum_{j\in M(\widetilde \bx)} \log q^j\bigl(x_j\mid \widetilde \bx,\tau_\ell\bigr)
    \right].
\end{equation}
We also define 
\begin{equation} \label{def:S_rec_ast}
    S_\rec^\ast(\bx)
    :=
    \frac{1}{L}
    \sum_{\ell=1}^L
    \mathbb E_{\widetilde \bx \sim q_{\tau_\ell|0}(\cdot\mid \bx)}
    \left[
    -\frac{
        1
    }{
        |M(\widetilde \bx)|\vee 1
    }
    \sum_{j\in M(\widetilde \bx)}
        \log p_{\hat\theta}^j\bigl(x_j\mid \widetilde \bx,\tau_\ell\bigr)
    \right],
\end{equation}
which is the same quantity as $S^\ast(\bx)$ except when the reconstruction model is imperfect.
If we define $\Delta_{\KL}(\bx) := S_\rec^\ast(\bx) - S^\ast(\bx)$, note that
\[
\mathbb E_{\bx \sim q_0}[\Delta_{\KL}(\bx)] = \frac{1}{L}
    \sum_{\ell=1}^L
    \mathbb E_{\widetilde \bx \sim q_{\tau_\ell}}
    \left[
    \frac{
        1
    }{
        |M(\widetilde \bx)|\vee 1
    }
    \sum_{j\in M(\widetilde \bx)}
        KL(q^j\bigl(\cdot \mid \widetilde \bx,\tau_\ell\bigr) || p_{\hat\theta}^j\bigl(\cdot \mid \widetilde \bx,\tau_\ell\bigr))
    \right] > 0.
\]
Further, when $p_{\hat\theta}^j$ is well-trained with the loss in Eq.~\eqref{eq:para_rec_loss_view_level}, we have $\mathbb E_{\bx \sim q_0}[\Delta_{\KL}(\bx)]\le \epsilon$ for some small $\epsilon$.
Finally, we let $S_\rec(\bx)$ be the anomaly score in Eq.~\eqref{eq:final_score}, which is a sampled version (or realization) of $S_\rec^\ast(\bx)$.
The following theorem summarizes the guarantee on Type-I/Type-II errors of the anomaly detector. The proof can be found in Appendix~\ref{app:theory}.

\begin{theorem}[Detection performance]
\label{thm:detection_performance}

We assume there exists constant $C>0$ such that for all \(j\), all probe levels \(\tau_\ell\), and all masked views
\(\widetilde \bx\), the trained and true models satisfy $-\log p_{\widehat{\theta}_{\mathrm{rec}}}^j(x_j\mid \widetilde \bx,\tau_\ell)\in[0,C]$, and $-\log q^{j}(x_j\mid \widetilde \bx,\tau_\ell)\in[0,C]$.
We also assume that the training loss is small such that $\mathbb E_{\bx \sim q_0}[\Delta_{\KL}(\bx)]\le \epsilon$. Then, for a fixed threshold $\gamma$, we have 
\begin{itemize}[leftmargin=1em]
    \item Type-I error: for normal data $\bx \sim q_0$, let $\mu_0^* := \mathbb E_{\bx \sim q_0}[S^\ast(\bx)]$ be the oracle normal mean. Then,
    \[
    \mathbb P_{\bx\sim q_0}(S_{\mathrm{rec}}(\bx) > \gamma)\leq 3 \exp\left(-\frac{\min\{L, K\} (\gamma - \mu_0^\ast - \epsilon)^2}{18 C^2}\right). 
    \]
    \item Type-II error: for anomalous data $\bx\sim q'
    $, let $\mu_1^* := \mathbb E_{q'}[S^*(\bx)]$.
    Assume that $\mu_1^*-\mu_0^\ast>0$ and no estimation error, i.e., $\varepsilon = 0$. Then,
    \[
    \mathbb{P}_{\bx \sim q'}(S_{\mathrm{rec}}(\bx) \le \gamma)
    \le
    2 \exp\left(-\frac{\min\{L,K\}(\mu_1^*-\gamma)^2}{2 C^2}\right).
    \]
\end{itemize}
\end{theorem}

\begin{remark}
The gap between $\mu_0^*$ and $\mu_1^*$ increases as $q'$ deviates further from $q_0$. To build some intuition, fix a large $\tau_\ell$ and consider the regime where $|M(\Tilde{\bx})|$ is also large and fixed. Then,
\begin{align*}
\mu_0^* &\approx
\mathbb E_{j} \mathbb E_{\bx \sim q_0, \widetilde \bx \sim q_{\tau_\ell|0}(\cdot\mid \bx)}
\left[
-\log q^j\bigl(x_j\mid \widetilde \bx,\tau_\ell\bigr)
\right]
= \mathbb E_{j} \mathbb E_{\widetilde \bx \sim q_{\tau_\ell}}
H\left(q^j\bigl(\cdot \mid \widetilde \bx,\tau_\ell\bigr)\right),
\\
\mu_1^* &\approx
\mathbb E_{j} \mathbb E_{\bx \sim q', \widetilde \bx \sim q_{\tau_\ell|0}(\cdot\mid \bx)}
\left[
-\log q^j\bigl(x_j\mid \widetilde \bx,\tau_\ell\bigr)
\right]
= \mathbb E_{j} \mathbb E_{\widetilde \bx \sim q'_{\tau_\ell}}
H\left(q'^j\bigl(\cdot \mid \widetilde \bx,\tau_\ell\bigr), q^j\bigl(\cdot \mid \widetilde \bx,\tau_\ell\bigr)\right),
\end{align*}
where $H(\cdot)$ denotes entropy, $q'_{\tau_\ell}$ is the forward masking distribution initialized from $q'$, and $H(q', q)$ denotes cross-entropy, defined as $H(q', q) = H(q') + KL(q' || q)$.
Since the diffusion process implies $q'_{\tau_\ell} \approx q_{\tau_\ell}$ for large $\tau_\ell$, the two terms become comparable. For such a case, the dominant contribution to $\mu_1^* - \mu_0^*$ arises from the KL divergence between the corresponding reconstruction distributions. This separation ensures reliable detection when $\gamma$ is chosen appropriately.



\end{remark}

\section{Numerical Experiments}

\paragraph{Numerical setup.} We evaluate \algname{} on \textit{eighteen} real-world datasets, including \textit{thirteen} fully discrete tabular data, \textit{one} mixed-type tabular data (with both categorical variables and continuous variables), and \textit{four} discrete text sequences. The summary of datasets is shown in Table~\ref{tab:datasets} in Appendix~\ref{sec:datasets}. 
All experiments follow the normal-only training of the anomaly detection protocol: during training, each method has access only to samples from the normal class; anomalous samples are used only for test-time evaluation. 
On tabular data, features are represented as discrete symbols, with continuous features discretized when necessary. On text data, each document is
converted into a fixed-length token sequence and masked reconstruction is performed at the token level. This allows the same forward-only scoring principle to be applied across different discrete data modalities. 

For \algname{}, we use fixed hyperparameters across datasets rather than
validation-based or per-dataset tuning. We use the standard linear absorbing
schedule $\alpha_t=1-t$, so the probe level $\tau_\ell$ has mask probability
$1-\alpha_{\tau_\ell}=\tau_\ell$. The parametric tabular and text models use
the uniform probe grid $\tau_\ell\in\{0.1,0.2,\ldots,0.9\}$, while the
non-parametric tabular model uses
$\tau_\ell\in\{0.15,0.30,0.45,0.60\}$.
For tabular data, the parametric reconstruction score averages over $K=16$ masked probe views at each probe level. The non-parametric model uses the
visible-coordinate Hamming kernel with bandwidth $\lambda=1$, takes all normal
training samples as the reference set, and averages over $8$ masked probe views
at each probe level. For text data, we use the GPT-2 \cite{radford2019language} tokenizer only, without
using pretrained GPT-2 weights, and we average the reconstruction score over $24$
independently generated masked probe views at each probe level. Model
architectures, optimization details, and additional implementation
hyperparameters are provided in Appendix~\ref{sec:hyperparameters}.


\vspace{-0.1in}
\paragraph{Baselines and metrics.} 

We compare against classical anomaly detection baselines, DTE-based variants, and recent tabular or text anomaly detection methods when applicable. Specifically, for tabular data, we compare against 12 baseline methods, including COPOD \cite{copod}, DeepSVDD \cite{pmlr-v80-ruff18a}, DTE-Categorical \cite{livernoche2024dte}, DTE-InvGamma \cite{livernoche2024dte}, DTE-Gaussian \cite{livernoche2024dte}, DRL \cite{ye2025drl}, ECOD \cite{9737003}, GOAD \cite{bergman2020goad}, Hamming \(k\)-nearest neighbors
(Hamming-\(k\)NN)~\cite{knn}, ICL \cite{shenkar2022anomaly}, IForest \cite{iforest}, and MCM \cite{yin2024mcm}. For text data, we compare with DATE~\cite{manolache-etal-2021-date}, FATE \cite{das2023fewshotanomalydetectiontext},
and several embedding-based baselines. Specifically, we use BERT
embeddings~\cite{devlin-etal-2019-bert} and OpenAI text
embeddings~\cite{openai_embeddings_2024} as fixed document representations,
followed by LOF \cite{lof}, DeepSVDD \cite{pmlr-v80-ruff18a}, ECOD \cite{9737003}, Isolation Forest \cite{iforest}, SO-GAAL \cite{sogaal}, AE \cite{aggarwal2017outlier}, VAE \cite{kingma2013autoencoding}, and LUNAR~\cite{Goodge_Hooi_Ng_Ng_2022}. We report ROC-AUC and PR-AUC: ROC-AUC measures global ranking quality, and PR-AUC summarizes the precision--recall trade-off and is especially informative under class imbalance. Detailed definitions are provided in Appendix~\ref{app:metric}.



\begin{table}[t]
\centering
\caption{Average ranks on tabular datasets over five random seeds. 
Ranks are computed separately for ROC-AUC and PR-AUC across datasets, and the overall rank is the average of the two metric-specific ranks.}
\label{tab:avg_ranks}
\vspace{-0.1in}
\begin{adjustbox}{max width=0.7\linewidth}
\begin{tabular}{lccc}
\toprule
Method & Overall Rank   & ROC-AUC Rank  & PR-AUC Rank   \\
\midrule
Parametric \algname{} & \textbf{3.929} & \textbf{3.714} & \textbf{4.143} \\
Non-parametric \algname{} & 4.929 & 4.857 & 5.000 \\
DTE-Gaussian & 5.286 & 4.929 & 5.643 \\
Hamming-kNN & 6.250 & 6.000 & 6.500 \\
DTE-Categorical & 6.357 & 5.214 & 7.500 \\
DRL & 7.429 & 7.857 & 7.000 \\
ECOD & 7.857 & 7.429 & 8.286 \\
COPOD & 7.929 & 7.857 & 8.000 \\
GOAD & 8.464 & 8.500 & 8.429 \\
DTE-InvGamma & 8.571 & 9.429 & 7.714 \\
ICL & 8.786 & 9.286 & 8.286 \\
MCM & 8.821 & 9.214 & 8.429 \\
DeepSVDD & 9.679 & 10.786 & 8.571 \\
IForest & 10.714 & 9.929 & 11.500 \\
\bottomrule
\end{tabular}
\end{adjustbox}
\end{table}

\vspace{-0.1in}
\paragraph{Results.} Table~\ref{tab:avg_ranks}, summarized from the results in
Tables~\ref{tab:rocauc_seed_avg} and~\ref{tab:prauc_seed_avg} in the Appendix, shows that our reconstruction-based variants are overall the strongest on tabular data. In particular, the parametric reconstruction anomaly score achieves the best average overall rank (\(3.93\)), and the non-parametric variant also ranks among the top methods. Among the baselines, \texttt{Hamming-kNN} is the strongest competitor, whereas \texttt{DTE-Gaussian} is the best-performing DTE variant. 
It is worthwhile mentioning that the DTE variants method compared here is implemented by embedding the discrete features into a continuous space and applying Gaussian DTE there. This represents a natural application of the DTE method in discrete data; yet this no longer defines a diffusion process on the original categorical space. The resulting time posterior depends on the geometry induced by the chosen embedding and may not faithfully preserve the distributional structure and inter-feature dependencies of normal data. In contrast, our proposed method operates directly in the discrete state space and achieves stronger overall performance, as reflected by its best average rank. This indicates that reconstruction-based scores are more reliable than time-based scores in the tabular setting. 
\begin{table}[h!]
\centering
\caption{Anomaly detection performance on NLP datasets.}
\label{tab:nlp_adbench_spam_results}
\begin{adjustbox}{max width=0.99\textwidth}
\begin{tabular}{lcccccccc}
\toprule
Method 
& \multicolumn{2}{c}{AGNews} 
& \multicolumn{2}{c}{Email Spam} 
& \multicolumn{2}{c}{SMS Spam} 
& \multicolumn{2}{c}{YelpReview}\\
\cmidrule(lr){2-3} \cmidrule(lr){4-5} \cmidrule(lr){6-7} \cmidrule(lr){8-9}
& ROC-AUC & PR-AUC 
& ROC-AUC & PR-AUC 
& ROC-AUC & PR-AUC 
& ROC-AUC & PR-AUC\\
\midrule
Parametric \algname{}
& 76.03 & 31.11 
& \textbf{97.55} & \textbf{88.87}
& \textbf{95.69} & \textbf{72.15} 
& 58.82 & 20.42 \\
\midrule
COPOD              & 50.91 & 11.46 & 53.46 & 18.31 & 86.77 & 31.76 & 54.54 & 18.98 \\
ECOD               & 49.19 & 11.06 & 55.61 & 19.39 & 86.07 & 31.31 & 54.65 & 19.02 \\
Hamming-kNN        & 50.64 & 11.80 & 56.68 & 27.11 & 87.26 & 32.70 & 58.55 & 21.76 \\
Isolation Forest   & 52.30 & 11.90 & 57.65 & 16.61 & 88.47 & 34.75 & 52.33 & 11.92 \\
\midrule
CVDD               & 60.46 & 12.96 & 93.40 & 53.53 & 47.82 & 7.12  & 53.45 & 17.11 \\
DATE               & 81.20 & 39.96 & 96.97 & 88.85 & 93.98 & 61.12 & 60.92 & 21.49 \\
FATE               & 77.56 & 27.87 & 90.61 & 55.29 & 62.62 & 12.57 & 59.45 & 21.12 \\
\midrule
BERT + AE          & 72.00 & 22.32 & 47.39 & 29.37 & 69.18 & 19.14 & 64.41 & 25.25 \\
BERT + DeepSVDD    & 66.71 & 21.60 & 69.37 & 21.17 & 58.59 & 11.78 & 58.71 & 21.74 \\
BERT + ECOD        & 63.18 & 16.16 & 70.52 & 20.77 & 56.06 & 11.56 & 63.26 & 21.97 \\
BERT + iForest     & 61.24 & 15.59 & 67.79 & 18.94 & 50.53 & 9.94  & 59.71 & 22.03 \\
BERT + LOF         & 74.32 & 25.49 & 74.82 & 23.70 & 71.90 & 18.37 & 65.73 & 26.29 \\
BERT + LUNAR       & 76.94 & 27.17 & 84.17 & 35.71 & 69.53 & 18.17 & 65.22 & 26.09 \\
BERT + SO-GAAL     & 44.89 & 10.33 & 44.40 & 11.30 & 33.28 & 7.14  & 47.12 & 24.40 \\
BERT + VAE         & 67.73 & 18.78 & 47.37 & 22.47 & 60.82 & 13.60 & 64.41 & 23.31 \\
\midrule
OpenAI + AE        & 83.26 & 40.22 & 76.51 & 55.80 & 55.11 & 10.30 & 85.24 & 70.63 \\
OpenAI + DeepSVDD  & 46.80 & 10.62 & 44.15 & 11.95 & 34.91 & 7.21  & 53.73 & 18.93 \\
OpenAI + ECOD      & 76.38 & 32.94 & 92.63 & 55.97 & 43.17 & 8.21  & 59.84 & \textbf{86.39} \\
OpenAI + iForest   & 52.13 & 12.78 & 69.37 & 32.83 & 37.51 & 7.72  & 58.71 & 25.27 \\
OpenAI + LOF       & 89.05 & 54.43 & 92.63 & 59.67 & 78.62 & 24.50 & 87.33 & 57.10 \\
OpenAI + LUNAR     & \textbf{92.26} & \textbf{69.18} & 93.43 & 58.10 & 71.89 & 16.40 & \textbf{94.52} & 45.24 \\
OpenAI + SO-GAAL   & 59.45 & 15.38 & 44.40 & 10.96 & 56.71 & 12.13 & 50.82 & 27.35 \\
OpenAI + VAE       & 81.44 & 36.59 & 52.73 & 56.04 & 42.59 & 8.12  & 61.63 & 84.67 \\
\bottomrule
\end{tabular}
\end{adjustbox}
\vspace{-0.03in}
\end{table}

For text data, Table~\ref{tab:nlp_adbench_spam_results} shows that the proposed parametric reconstruction score is especially effective on short
spam-detection datasets. It achieves the best performance on Email Spam and SMS Spam, outperforming both end-to-end text anomaly detectors such as DATE and embedding-based baselines. On AGNews and YelpReview, however, our method is less competitive than the strongest OpenAI-embedding baselines, suggesting that long or semantically richer text sequences may require more specialized sequence architectures or text-specific training. Overall, these results indicate that masked reconstruction serves as a promising signal for identifying anomalous discrete sequences, while also highlighting an important direction for improving \algname{} on longer text data.

\begin{figure}[ht!]
\centering 
\includegraphics[width=0.6\linewidth]{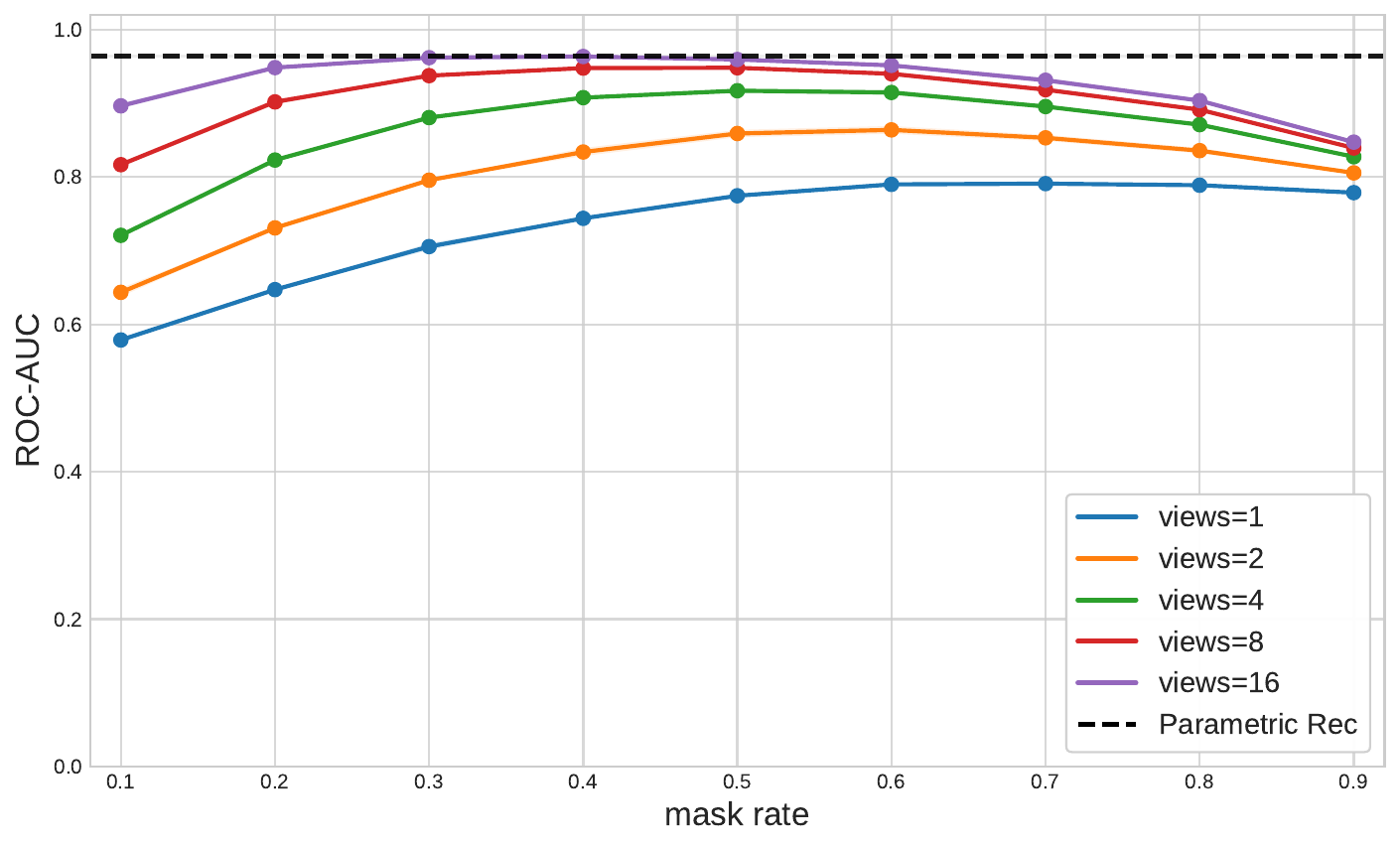}
\vspace{-0.1in}
    \caption{Sensitivity of the ROC-AUC to the probe mask rate on the Vehicle Claims dataset.
    Each solid curve uses a single probe mask rate $\tau$ and averages over
    $K\in\{1,2,4,8,16\}$ masked views.
    The dashed line denotes our default multi-probe parametric reconstruction score, which aggregates over the full uniform probe grid.
    }
    \label{fig:probe_tau_sensitivity}   
\end{figure}

\vspace{-0.1in}
\paragraph{Sensitivity analysis on probe level $\tau$.}
We further study the effect of the probe level using the Vehicle Claims dataset. Figure~\ref{fig:probe_tau_sensitivity} reports the ROC-AUC of the proposed parametric score when computed from a single probe mask rate $\tau$.
The ROC-AUC first increases at small $\tau$ 
and then decreases when too much context is removed.
This suggests a trade-off: at small $\tau$ the algorithm masks too few coordinates to expose anomalous dependencies, while at large $\tau$ the missing context makes reconstruction difficult for both normal and anomalous samples.
Additional results in Appendix~\ref{app:sensitivity_analysis} show that the best single probe level can differ across datasets.
Therefore, we use a fixed multi-probe score in the main experiments, aggregating over a uniform grid of probe levels instead of tuning $\tau$ separately for each dataset.

\section{Conclusion and Discussions}

We proposed \algname{}, a masked diffusion-based anomaly detection framework for discrete data. The method uses a trained masked diffusion model for anomaly scoring: a test sample is
randomly masked at multiple probe levels, and its anomaly score is the average reconstruction surprisal of the masked entries. This design directly leverages the discrete structure of the data and is computationally efficient to compute. Our theoretical analysis provided Type-I and Type-II error insights under well-specified reconstruction models, and our experiments showed competitive
performance across eighteen tabular and text datasets. One limitation is that the method depends on the quality of the learned masked
conditionals and on the choice of probe levels, although multi-probe averaging
helps to reduce this sensitivity. In addition, its performance on long text
sequences is currently less competitive, which may be partly due to the simple text modeling and token-level architecture used in our experiments. Future work includes
adaptive selection of probe levels, text-specific architectures and training
strategies for long sequences, extensions to structured discrete objects such as graphs and molecules, and tighter theory connecting masked reconstruction with distributional divergences between nominal and anomalous data.


\bibliographystyle{unsrt}
\bibliography{ref}

@inproceedings{
shi2024simplified,
title={Simplified and Generalized Masked Diffusion for Discrete Data},
author={Jiaxin Shi and Kehang Han and Zhe Wang and Arnaud Doucet and Michalis Titsias},
booktitle={The Thirty-eighth Annual Conference on Neural Information Processing Systems},
year={2024},
url={https://openreview.net/forum?id=xcqSOfHt4g}
}

@inproceedings{
prabhudesai2026diffusion,
title={Diffusion Beats Autoregressive in Data-Constrained Settings},
author={Mihir Prabhudesai and Mengning Wu and Amir Zadeh and Katerina Fragkiadaki and Deepak Pathak},
booktitle={The Thirty-ninth Annual Conference on Neural Information Processing Systems},
year={2025},
url={https://openreview.net/forum?id=W5Ht05jF4c}
}

@article{cardei2025constrained,
	author = {Michael Cardei and Jacob K Christopher and Thomas Hartvigsen and Bhavya Kailkhura and Ferdinando Fioretto},
	journal = {arXiv preprint arXiv:2503.09790},
	title = {Constrained Discrete Diffusion},
	year = {2025}}

@inproceedings{salem2013sensor,
  title={Sensor fault and patient anomaly detection and classification in medical wireless sensor networks},
  author={Salem, Osman and Guerassimov, Alexey and Mehaoua, Ahmed and Marcus, Anthony and Furht, Borko},
  booktitle={2013 IEEE International Conference on Communications (ICC)},
  pages={4373--4378},
  year={2013},
  organization={IEEE}
}

@inproceedings{yairi2006telemetry,
  title={Telemetry-mining: a machine learning approach to anomaly detection and fault diagnosis for space systems},
  author={Yairi, Takehisa and Kawahara, Yoshinobu and Fujimaki, Ryohei and Sato, Yuichi and Machida, Kazuo},
  booktitle={2nd IEEE International Conference on Space Mission Challenges for Information Technology (SMC-IT'06)},
  pages={8--pp},
  year={2006},
  organization={IEEE}
}

@article{fraser2022challenges,
  title={Challenges for unsupervised anomaly detection in particle physics},
  author={Fraser, Katherine and Homiller, Samuel and Mishra, Rashmish K and Ostdiek, Bryan and Schwartz, Matthew D},
  journal={Journal of High Energy Physics},
  volume={2022},
  number={3},
  pages={66},
  year={2022},
  publisher={Springer}
}

@article{susto2017anomaly,
  title={Anomaly detection approaches for semiconductor manufacturing},
  author={Susto, Gian Antonio and Terzi, Matteo and Beghi, Alessandro},
  journal={Procedia Manufacturing},
  volume={11},
  pages={2018--2024},
  year={2017},
  publisher={Elsevier}
}

@article{ahmed2016survey,
  title={A survey of anomaly detection techniques in financial domain},
  author={Ahmed, Mohiuddin and Mahmood, Abdun Naser and Islam, Md Rafiqul},
  journal={Future Generation Computer Systems},
  volume={55},
  pages={278--288},
  year={2016},
  publisher={Elsevier}
}

@article{pachauri2015anomaly,
  title={Anomaly detection in medical wireless sensor networks using machine learning algorithms},
  author={Pachauri, Girik and Sharma, Sandeep},
  journal={Procedia Computer Science},
  volume={70},
  pages={325--333},
  year={2015},
  publisher={Elsevier}
}

@inproceedings{zheng2025fhs,
	author = {Zheng, Kaiwen and Chen, Yongxin and Mao, Hanzi and Liu, Ming-Yu and Zhu, Jun and Zhang, Qinsheng},
	booktitle = {International Conference on Learning Representations (ICLR)},
	title = {Masked Diffusion Models Are Secretly Time-Agnostic Masked Models and Exploit Inaccurate Categorical Sampling},
	year = {2025}}

@article{uehara2024survey-finetuning,
	author = {Masatoshi Uehara and Yulai Zhao and Tommaso Biancalani and Sergey Levine},
	journal = {arXiv preprint arXiv:2407.13734},
	title = {Understanding Reinforcement Learning-Based Fine-Tuning of Diffusion Models: A Tutorial and Review},
	year = {2024}}

@inproceedings{
livernoche2024dte,
title={On Diffusion Modeling for Anomaly Detection},
author={Victor Livernoche and Vineet Jain and Yashar Hezaveh and Siamak Ravanbakhsh},
booktitle={The Twelfth International Conference on Learning Representations},
year={2024},
url={https://openreview.net/forum?id=lR3rk7ysXz}
}

@article{han2022adbench,
  title={{Adbench}: Anomaly detection benchmark},
  author={Han, Songqiao and Hu, Xiyang and Huang, Hailiang and Jiang, Minqi and Zhao, Yue},
  journal={Advances in Neural Information Processing Systems},
  volume={35},
  pages={32142--32159},
  year={2022}
}

@article{sahoo2024simple,
  title={Simple and effective masked diffusion language models},
  author={Sahoo, Subham S and Arriola, Marianne and Schiff, Yair and Gokaslan, Aaron and Marroquin, Edgar and Chiu, Justin T and Rush, Alexander and Kuleshov, Volodymyr},
  journal={Advances in Neural Information Processing Systems},
  volume={37},
  pages={130136--130184},
  year={2024}
}

@article{austin2021structured,
  title={Structured denoising diffusion models in discrete state-spaces},
  author={Austin, Jacob and Johnson, Daniel D and Ho, Jonathan and Tarlow, Daniel and Van Den Berg, Rianne},
  journal={Advances in Neural Information Processing Systems},
  volume={34},
  pages={17981--17993},
  year={2021}
}

@article{ho2020denoising,
  title={Denoising diffusion probabilistic models},
  author={Ho, Jonathan and Jain, Ajay and Abbeel, Pieter},
  journal={Advances in Neural Information Processing Systems},
  volume={33},
  pages={6840--6851},
  year={2020}
}

@inproceedings{sohl2015deep,
  title={Deep unsupervised learning using nonequilibrium thermodynamics},
  author={Sohl-Dickstein, Jascha and Weiss, Eric and Maheswaranathan, Niru and Ganguli, Surya},
  booktitle={International conference on machine learning},
  pages={2256--2265},
  year={2015},
  organization={pmlr}
}

@article{taha2019anomaly,
  title={Anomaly detection methods for categorical data: A review},
  author={Taha, Ayman and Hadi, Ali S},
  journal={ACM Computing Surveys (CSUR)},
  volume={52},
  number={2},
  pages={1--35},
  year={2019},
  publisher={ACM New York, NY, USA}
}

@article{ruff2021unifying,
  title={A unifying review of deep and shallow anomaly detection},
  author={Ruff, Lukas and Kauffmann, Jacob R and Vandermeulen, Robert A and Montavon, Gr{\'e}goire and Samek, Wojciech and Kloft, Marius and Dietterich, Thomas G and M{\"u}ller, Klaus-Robert},
  journal={Proceedings of the IEEE},
  volume={109},
  number={5},
  pages={756--795},
  year={2021},
  publisher={IEEE}
}

@article{chandola2009anomaly,
  title={Anomaly detection: A survey},
  author={Chandola, Varun and Banerjee, Arindam and Kumar, Vipin},
  journal={ACM computing surveys (CSUR)},
  volume={41},
  number={3},
  pages={1--58},
  year={2009},
  publisher={ACM New York, NY, USA}
}

@article{chawda2022uadad,
  title={Unsupervised anomaly detection for auditing data and impact of categorical encodings},
  author={Chawda, Ajay and Grimm, Stefanie and Kloft, Marius},
  journal={arXiv preprint arXiv:2210.14056},
  year={2022}
}

@article{li2025nlp,
  title={Nlp-adbench: Nlp anomaly detection benchmark},
  author={Li, Yuangang and Li, Jiaqi and Xiao, Zhuo and Yang, Tiankai and Nian, Yi and Hu, Xiyang and Zhao, Yue},
  journal={Findings of the Association for Computational Linguistics: EMNLP 2025},
  year={2025}
}

@article{pang2021deep,
  title={Deep learning for anomaly detection: A review},
  author={Pang, Guansong and Shen, Chunhua and Cao, Longbing and Hengel, Anton Van Den},
  journal={ACM Computing Surveys (CSUR)},
  volume={54},
  number={2},
  pages={1--38},
  year={2021},
  publisher={ACM New York, NY, USA}
}

@ARTICLE{9737003,
  author={Li, Zheng and Zhao, Yue and Hu, Xiyang and Botta, Nicola and Ionescu, Cezar and Chen, George H.},
  journal={IEEE Transactions on Knowledge and Data Engineering}, 
  title={{ECOD}: Unsupervised Outlier Detection Using Empirical Cumulative Distribution Functions}, 
  year={2023},
  volume={35},
  number={12},
  pages={12181-12193},
  doi={10.1109/TKDE.2022.3159580}}

@article{knn,
title = "Efficient algorithms for mining outliers from large data sets",
author = "Sridhar Ramaswamy and Rajeev Rastogi and Kyuseok Shim",
year = "2000",
month = jun,
doi = "10.1145/335191.335437",
language = "English",
volume = "29",
pages = "427--438",
journal = "SIGMOD Record",
issn = "0163-5808",
publisher = "Association for Computing Machinery (ACM)",
number = "2",
}

@inproceedings{iforest, author = {Liu, Fei Tony and Ting, Kai Ming and Zhou, Zhi-Hua}, title = {Isolation Forest}, year = {2008}, isbn = {9780769535029}, publisher = {IEEE Computer Society}, address = {USA}, url = {https://doi.org/10.1109/ICDM.2008.17}, doi = {10.1109/ICDM.2008.17}, booktitle = {Proceedings of the 2008 Eighth IEEE International Conference on Data Mining}, pages = {413–422}, numpages = {10}, series = {ICDM '08} }

@inproceedings{copod,
   title={COPOD: Copula-Based Outlier Detection},
   url={http://dx.doi.org/10.1109/ICDM50108.2020.00135},
   DOI={10.1109/icdm50108.2020.00135},
   booktitle={2020 IEEE International Conference on Data Mining (ICDM)},
   publisher={IEEE},
   author={Li, Zheng and Zhao, Yue and Botta, Nicola and Ionescu, Cezar and Hu, Xiyang},
   year={2020},
   month=nov, pages={1118–1123} }

@inproceedings{manolache-etal-2021-date,
    title = "{DATE}: Detecting Anomalies in Text via Self-Supervision of Transformers",
    author = "Manolache, Andrei  and
      Brad, Florin  and
      Burceanu, Elena",
    editor = "Toutanova, Kristina  and
      Rumshisky, Anna  and
      Zettlemoyer, Luke  and
      Hakkani-Tur, Dilek  and
      Beltagy, Iz  and
      Bethard, Steven  and
      Cotterell, Ryan  and
      Chakraborty, Tanmoy  and
      Zhou, Yichao",
    booktitle = "Proceedings of the 2021 Conference of the North American Chapter of the Association for Computational Linguistics: Human Language Technologies",
    month = jun,
    year = "2021",
    address = "Online",
    publisher = "Association for Computational Linguistics",
    url = "https://aclanthology.org/2021.naacl-main.25/",
    doi = "10.18653/v1/2021.naacl-main.25",
    pages = "267--277"
}

@article{Goodge_Hooi_Ng_Ng_2022, title={LUNAR: Unifying Local Outlier Detection Methods via Graph Neural Networks}, volume={36}, url={https://ojs.aaai.org/index.php/AAAI/article/view/20629}, DOI={10.1609/aaai.v36i6.20629}, abstractNote={Many well-established anomaly detection methods use the distance of a sample to those in its local neighbourhood: so-called `local outlier methods’, such as LOF and DBSCAN. They are popular for their simple principles and strong performance on unstructured, feature-based data that is commonplace in many practical applications. However, they cannot learn to adapt for a particular set of data due to their lack of trainable parameters. In this paper, we begin by unifying local outlier methods by showing that they are particular cases of the more general message passing framework used in graph neural networks. This allows us to introduce learnability into local outlier methods, in the form of a neural network, for greater flexibility and expressivity: specifically, we propose LUNAR, a novel, graph neural network-based anomaly detection method. LUNAR learns to use information from the nearest neighbours of each node in a trainable way to find anomalies. We show that our method performs significantly better than existing local outlier methods, as well as state-of-the-art deep baselines. We also show that the performance of our method is much more robust to different settings of the local neighbourhood size.}, number={6}, journal={Proceedings of the AAAI Conference on Artificial Intelligence}, author={Goodge, Adam and Hooi, Bryan and Ng, See-Kiong and Ng, Wee Siong}, year={2022}, month={Jun.}, pages={6737-6745} }

@inproceedings{devlin-etal-2019-bert,
    title = "{BERT}: Pre-training of Deep Bidirectional Transformers for Language Understanding",
    author = "Devlin, Jacob  and
      Chang, Ming-Wei  and
      Lee, Kenton  and
      Toutanova, Kristina",
    editor = "Burstein, Jill  and
      Doran, Christy  and
      Solorio, Thamar",
    booktitle = "Proceedings of the 2019 Conference of the North {A}merican Chapter of the Association for Computational Linguistics: Human Language Technologies, Volume 1 (Long and Short Papers)",
    month = jun,
    year = "2019",
    address = "Minneapolis, Minnesota",
    publisher = "Association for Computational Linguistics",
    url = "https://aclanthology.org/N19-1423/",
    doi = "10.18653/v1/N19-1423",
    pages = "4171--4186"
}

@misc{openai_embeddings_2024,
  author       = {{OpenAI}},
  title        = {New embedding models and API updates},
  howpublished = {\url{https://openai.com/index/new-embedding-models-and-api-updates/}}
}

@inproceedings{
yin2024mcm,
title={{MCM}: Masked Cell Modeling for Anomaly Detection in Tabular Data},
author={Jiaxin Yin and Yuanyuan Qiao and Zitang Zhou and Xiangchao Wang and Jie Yang},
booktitle={The Twelfth International Conference on Learning Representations},
year={2024},
url={https://openreview.net/forum?id=lNZJyEDxy4}
}

@InProceedings{pmlr-v235-thimonier24a,
    title = {Beyond Individual Input for Deep Anomaly Detection on Tabular Data},
    author = {Thimonier, Hugo and Popineau, Fabrice and Rimmel, Arpad and Doan, Bich-Li\^{e}n},
    booktitle = {Proceedings of the 41st International Conference on Machine Learning},
    pages = {48097--48123},
    year = {2024},
    volume = {235},
    series = {Proceedings of Machine Learning Research},
    month =  {21--27 Jul},
    publisher = {PMLR},
    }

@inproceedings{
shenkar2022anomaly,
title={Anomaly Detection for Tabular Data with Internal Contrastive Learning},
author={Tom Shenkar and Lior Wolf},
booktitle={International Conference on Learning Representations},
year={2022},
url={https://openreview.net/forum?id=_hszZbt46bT}
}

@inproceedings{comprex,
author = {Akoglu, Leman and Tong, Hanghang and Vreeken, Jilles and Faloutsos, Christos},
title = {Fast and reliable anomaly detection in categorical data},
year = {2012},
isbn = {9781450311564},
publisher = {Association for Computing Machinery},
address = {New York, NY, USA},
url = {https://doi.org/10.1145/2396761.2396816},
doi = {10.1145/2396761.2396816},
booktitle = {Proceedings of the 21st ACM International Conference on Information and Knowledge Management},
pages = {415–424},
numpages = {10},
location = {Maui, Hawaii, USA},
series = {CIKM '12}
}

@inproceedings{
ye2025drl,
title={{DRL}: Decomposed Representation Learning for Tabular Anomaly Detection},
author={Hangting Ye and He Zhao and Wei Fan and Mingyuan Zhou and Dan dan Guo and Yi Chang},
booktitle={The Thirteenth International Conference on Learning Representations},
year={2025},
url={https://openreview.net/forum?id=CJnceDksRd}
}

@inproceedings{lof,
  title={LOF: identifying density-based local outliers},
  author={Breunig, Markus M and Kriegel, Hans-Peter and Ng, Raymond T and Sander, J{\"o}rg},
  booktitle={Proceedings of the 2000 ACM SIGMOD international conference on Management of data},
  pages={93--104},
  year={2000}
}

@InProceedings{pmlr-v80-ruff18a,
  title = 	 {Deep One-Class Classification},
  author =       {Ruff, Lukas and Vandermeulen, Robert and Goernitz, Nico and Deecke, Lucas and Siddiqui, Shoaib Ahmed and Binder, Alexander and M{\"u}ller, Emmanuel and Kloft, Marius},
  booktitle = 	 {Proceedings of the 35th International Conference on Machine Learning},
  pages = 	 {4393--4402},
  year = 	 {2018},
  editor = 	 {Dy, Jennifer and Krause, Andreas},
  volume = 	 {80},
  series = 	 {Proceedings of Machine Learning Research},
  month = 	 {10--15 Jul},
  publisher =    {PMLR},
  url = 	 {https://proceedings.mlr.press/v80/ruff18a.html}
}

@article{sogaal,
  title={Generative adversarial active learning for unsupervised outlier detection},
  author={Liu, Yezheng and Li, Zhe and Zhou, Chong and Jiang, Yuanchun and Sun, Jianshan and Wang, Meng and He, Xiangnan},
  journal={IEEE Transactions on Knowledge and Data Engineering},
  volume={32},
  number={8},
  pages={1517--1528},
  year={2019},
  publisher={IEEE}
}

@inproceedings{das2023fewshotanomalydetectiontext,
  title={Few-shot anomaly detection in text with deviation learning},
  author={Das, Anindya Sundar and Ajay, Aravind and Saha, Sriparna and Bhuyan, Monowar},
  booktitle={International Conference on Neural Information Processing},
  pages={425--438},
  year={2023},
  organization={Springer}
}

@article{kingma2013autoencoding,
  title         = {Auto-Encoding Variational Bayes},
  author        = {Kingma, Diederik P. and Welling, Max},
  journal       = {arXiv preprint arXiv:1312.6114},
  year          = {2013},
  eprint        = {1312.6114},
  archivePrefix = {arXiv},
  primaryClass  = {stat.ML},
  doi           = {10.48550/arXiv.1312.6114}
}

@book{aggarwal2017outlier,
  title     = {Outlier Analysis},
  author    = {Aggarwal, Charu C.},
  edition   = {2},
  publisher = {Springer},
  year      = {2017},
  doi       = {10.1007/978-3-319-47578-3}
}

@inproceedings{bergman2020goad,
  author    = {Liron Bergman and Yedid Hoshen},
  title     = {Classification-Based Anomaly Detection for General Data},
  booktitle = {International Conference on Learning Representations (ICLR)},
  year      = {2020}
}

@article{radford2019language,
  title={Language models are unsupervised multitask learners},
  author={Radford, Alec and Wu, Jeffrey and Child, Rewon and Luan, David and Amodei, Dario and Sutskever, Ilya and others},
  journal={OpenAI blog},
  volume={1},
  number={8},
  pages={9},
  year={2019}
}

@inproceedings{wyatt2022anoddpm,
  title={Anoddpm: Anomaly detection with denoising diffusion probabilistic models using simplex noise},
  author={Wyatt, Julian and Leach, Adam and Schmon, Sebastian M and Willcocks, Chris G},
  booktitle={Proceedings of the IEEE/CVF conference on computer vision and pattern recognition},
  pages={650--656},
  year={2022}
}

@article{liu2025survey,
  title={A survey on diffusion models for anomaly detection},
  author={Liu, Jing and Ma, Zhenchao and Wang, Zepu and Zou, Chenxuanyin and Ren, Jiayang and Wang, Zehua and Song, Liang and Hu, Bo and Liu, Yang and Leung, Victor},
  journal={arXiv preprint arXiv:2501.11430},
  year={2025}
}

@inproceedings{sakurada2014anomaly,
  title={Anomaly detection using autoencoders with nonlinear dimensionality reduction},
  author={Sakurada, Mayu and Yairi, Takehisa},
  booktitle={Proceedings of the MLSDA 2014 2nd workshop on machine learning for sensory data analysis},
  pages={4--11},
  year={2014}
}

@article{an2015variational,
  title={Variational autoencoder based anomaly detection using reconstruction probability},
  author={An, Jinwon and Cho, Sungzoon},
  journal={Special lecture on IE},
  volume={2},
  number={1},
  pages={1--18},
  year={2015}
}

@inproceedings{zong2018deep,
  title={Deep autoencoding gaussian mixture model for unsupervised anomaly detection},
  author={Zong, Bo and Song, Qi and Min, Martin Renqiang and Cheng, Wei and Lumezanu, Cristian and Cho, Daeki and Chen, Haifeng},
  booktitle={International Conference on Learning Representations},
  year={2018}
}

@inproceedings{schlegl2017anogan,
  title={Unsupervised Anomaly Detection with Generative Adversarial Networks to Guide Marker Discovery},
  author={Schlegl, Thomas and Seeb{\"o}ck, Philipp and Waldstein, Sebastian M. and Schmidt-Erfurth, Ursula and Langs, Georg},
  booktitle={Information Processing in Medical Imaging},
  pages={146--157},
  year={2017},
  publisher={Springer}
}

@inproceedings{akcay2018ganomaly,
  title={Ganomaly: Semi-supervised anomaly detection via adversarial training},
  author={Akcay, Samet and Atapour-Abarghouei, Amir and Breckon, Toby P},
  booktitle={Asian conference on computer vision},
  pages={622--637},
  year={2018},
  organization={Springer}
}

@inproceedings{mousakhan2024anomaly,
  title={Anomaly detection with conditioned denoising diffusion models},
  author={Mousakhan, Arian and Brox, Thomas and Tayyub, Jawad},
  booktitle={DAGM German Conference on Pattern Recognition},
  pages={181--195},
  year={2024},
  organization={Springer}
}

@article{zhang2025diffusionad,
  title={DiffusionAD: Norm-guided one-step denoising diffusion for anomaly detection},
  author={Zhang, Hui and Wang, Zheng and Zeng, Dan and Wu, Zuxuan and Jiang, Yu-Gang},
  journal={IEEE transactions on pattern analysis and machine intelligence},
  year={2025},
  publisher={IEEE}
}

@inproceedings{rudolph2021same,
  title={Same same but differnet: Semi-supervised defect detection with normalizing flows},
  author={Rudolph, Marco and Wandt, Bastian and Rosenhahn, Bodo},
  booktitle={Proceedings of the IEEE/CVF winter conference on applications of computer vision},
  pages={1907--1916},
  year={2021}
}

\appendix

\section{Algorithmic Details}\label{app:alg}

In this section, we provide the detailed algorithms for the two variants mentioned in Section~\ref{sec:nonpara}. Algorithm~\ref{alg:rec_scoring_nonpara} summarizes the non-parametric version of the \algname{} using the non-parametric score computed from Eq.~\eqref{eq:nonpara}. Algorithm~\ref{alg:rec_training} summarizes the training procedure when using the alternative training objective in Eq.~\eqref{eq:para_rec_loss_view_level}.

\begin{algorithm}[ht!]
\caption{Non-Parametric \algname{}}
\label{alg:rec_scoring_nonpara}
\begin{algorithmic}[1]
\Require Test sample \(\bx\); normal-only training data \(\mathcal D_{\rm tr}\); probe grids \(\{\tau_\ell\}_{\ell=1}^L\); number of views per mask level \(K\); detection threshold \(\gamma\).
\Ensure Anomaly label $\hat y(\bx)\in\{0,1\}$, where 1 indicates anomalous and 0 indicates normal.
\For{each \(\ell=1,\ldots,L\) and \(k=1,\ldots,K\)}
        \State Sample the corresponding probe-masked view
        $\widetilde \bx^{(k,\ell)}\sim q_{\tau_\ell|0}(\cdot\mid \bx)$ according to Eq.~\eqref{eq:probe}.
        \State Compute per-view reconstruction difficulty
        \[
        s^{\rm NP}_{\rm rec}\bigl(\bx;\widetilde \bx^{(\ell,k)}\bigr) = -\frac{
        1}{
        |M(\widetilde \bx^{(\ell,k)})|\vee 1} \sum_{j\in M(\widetilde \bx^{(\ell,k)})}
        \log \widehat p_{\rm NP}^j(x_j\mid \widetilde \bx^{(\ell,k)})
        \]
        where $\widehat p_{\mathrm{NP}}^j(\cdot \mid \widetilde \bx)$ is as in Eq.~\eqref{eq:nonpara}.
\EndFor
\State Compute the aggregated anomaly score 
\[
S^{\rm NP}_\rec(\bx)
    :=
    \frac{1}{LK}
    \sum_{\ell=1}^L
    \sum_{k=1}^K
    s^{\rm NP}_{\rm rec}\bigl(\bx;\widetilde \bx^{(\ell,k)}\bigr).
\]
\State \Return Anomaly label $\hat y(\bx)=\mathbf 1\{S^{\rm NP}_\rec(\bx)>\gamma\}$.
\end{algorithmic}
\end{algorithm}

\begin{algorithm}[ht!]
\caption{Training of Alternative Parametric Reconstruction Model}
\label{alg:rec_training}
\begin{algorithmic}[1]
\Require Normal training set \(D_{\mathrm{tr}}=\{\bx^{(n)}\}_{n=1}^N\), probe grids rates \(\{\tau_\ell\}_{\ell=1}^L\), total epochs \(E\), learning rate \(\eta\).
\Ensure Trained reconstruction model \(p_{\widehat{\theta}_{\mathrm{rec}}}^j(\cdot \mid \widetilde \bx),j=1,\ldots,d.\)

\State Initialize reconstruction model
$p_{\widehat{\theta}_{\mathrm{rec}}}^j$

\For{\(\text{epoch}=1,\ldots,E\)}
    \For{each mask level \(\ell=1,\ldots,L\)}
        \State Sample masked views 
        $\widetilde{\bx}^{(\ell,n)} \sim q_{\tau_\ell|0}(\cdot\mid \bx^{(n)}),n=1,\ldots,N$
    \EndFor

    \State Compute the reconstruction loss $\mathcal L_{\mathrm{rec}}(\theta_{\mathrm{rec}})$ according to \eqref{eq:para_rec_loss_view_level}

    \State Update parameters 
    $\theta_{\mathrm{rec}}
    \gets
    \theta_{\mathrm{rec}}
    -
    \eta\nabla_{\theta_{\mathrm{rec}}}
    \mathcal L_{\mathrm{rec}}(\theta_{\mathrm{rec}})$
\EndFor
\end{algorithmic}
\end{algorithm}

\section{Experiment Details}

\subsection{Datasets and Pre-processing}\label{sec:datasets}

Table~\ref{tab:datasets} summarizes the eighteen datasets we used for evaluating the proposed method. 
The datasets include thirteen categorical tabular datasets from ADBench~\cite{pang2021deep},
one mixed-type vehicle claims dataset from UADAD~\cite{chawda2022uadad},
and four text anomaly detection datasets from NLP-ADBench~\cite{li2025nlp}.
All experiments follow the normal-only anomaly detection protocol: anomaly labels are used only
for constructing stratified splits and for test-time evaluation, but are never used for training the
proposed method.


\begin{table}[ht!] 
\centering
\caption{Datasets overview.}
\label{tab:datasets}
\begin{adjustbox}{max width=0.75\linewidth}
\begin{tabular}{lcccc}
\toprule
Dataset Name & \# Samples & \# Features & \% Anomaly & Data Type \\
\midrule
AD \cite{pang2021deep} & 3,279 & 1,555 & 14.00\% & Categorical\\
AID362 \cite{pang2021deep} & 4,279 & 114 & 1.40\% & Categorical\\
APAS \cite{pang2021deep} & 12,695 & 64 & 1.39\% & Categorical\\
bank \cite{pang2021deep} & 41,188 & 10 & 11.27\% & Categorical\\
census \cite{pang2021deep} & 299,285 & 33 & 6.20\% & Categorical\\
Chess \cite{pang2021deep} & 28,056 & 6 & 0.10\% & Categorical\\
CMC \cite{pang2021deep} & 1,473 & 8 & 1.97\% & Categorical\\
CoverType \cite{pang2021deep} & 581,012 & 44 & 0.47\% & Categorical\\
probe \cite{pang2021deep} & 64,759 & 6 & 6.43\% & Categorical\\
R10 \cite{pang2021deep} & 12,897 & 100 & 1.84\% & Categorical\\
Solar \cite{pang2021deep} & 1,066 & 11 & 4.03\% & Categorical\\
U2R \cite{pang2021deep} & 60,821 & 6 & 0.38\% & Categorical\\
Vehicle claims \cite{chawda2022uadad}& 268,255 & 18 & 21.15\% & Mixed-type \\
w7a \cite{pang2021deep} & 49,749 & 300 & 2.97\% & Categorical\\

\midrule
AGNews \cite{li2025nlp}  & 98,207 & 128 & 3.85\% & NLP \\
Email spam \cite{li2025nlp} & 3,578 & 256 & 4.08\% & NLP \\
SMS spam \cite{li2025nlp} & 4,672 & 128 & 3.30\% & NLP \\
YelpReview \cite{li2025nlp} & 316,924 & 256 & 5.66\% & NLP \\
\bottomrule
\end{tabular}
\end{adjustbox}
\end{table}

\paragraph{UADAD.} 
We use the labeled vehicle claims dataset from the UADAD benchmark~\cite{chawda2022uadad}.
Following the main feature setting of that benchmark, we retain the following attributes:
\[
\begin{gathered}
\texttt{Maker},\ \texttt{Genmodel},\ \texttt{Color},\ \texttt{Reg\_year},\ 
\texttt{Bodytype},\ \texttt{Runned\_Miles},\\
\texttt{Engin\_size},\ \texttt{Gearbox},\ \texttt{Fuel\_type},\ 
\texttt{Price},\ \texttt{Seat\_num},\ \texttt{Door\_num},\\
\texttt{issue},\ \texttt{issue\_id},\ \texttt{repair\_complexity},\ 
\texttt{repair\_hours},\ \texttt{repair\_cost}.
\end{gathered}
\]
The column \texttt{Label} is used only for splitting and evaluation, where
\(\texttt{Label}=0\) denotes a normal sample and \(\texttt{Label}=1\) denotes an anomalous sample.
We randomly split the full labeled dataset into \(70\%\) training and \(30\%\) testing
using a stratified split with respect to \texttt{Label}. To match the unsupervised anomaly detection setting, only normal samples with \(\texttt{Label}=0\) in the training split are used to train the model or build the reference set. The full test split, containing both normal and anomalous samples, is used for evaluation.

We treat the vehicle claims dataset as a mixed-type tabular dataset. Categorical attributes are kept as discrete symbols and encoded into integer IDs. Numerical attributes are discretized by quantile binning, where the bin boundaries are fitted using only the normal training samples and then applied to the test set. This yields a fully discrete representation for \algname{} and the corresponding discrete baselines.


\paragraph{ADBench.} 
For the ADBench datasets~\cite{pang2021deep}, we also randomly split each full labeled dataset into \(70\%\) training and \(30\%\) testing using a stratified split with respect to the anomaly label. After the split, all anomalous samples in the training portion are removed, and only nominal training samples are used for model training or for constructing the non-parametric reference set. The test set is kept unchanged and contains both normal and anomalous samples. Thus, anomalous samples are never observed during training, and labels are used only for evaluation.

All ADBench datasets used in our experiments are treated as categorical tabular datasets. Each feature is represented as a discrete symbol and encoded into an integer ID. The same discrete representation is used for MaskDiff-AD and for tabular baselines whenever applicable.

\paragraph{NLP-ADBench.} 

For the NLP anomaly detection datasets, we follow the official train/test split provided by NLP-ADBench~\cite{li2025nlp}, which uses \(90\%\) of the data for training and \(10\%\) for testing. As in the tabular experiments, only normal samples in the training split are used for training. The full test split, containing both normal and anomalous text samples, is used for computing anomaly scores and evaluation metrics.

For text preprocessing, we convert each raw text sample into a fixed-length
discrete sequence using the GPT-2 tokenizer~\cite{radford2019language}. Our
model does not use the pretrained GPT-2 language model or its pretrained token embeddings; the tokenizer is used only to map raw text into discrete token IDs.

Specifically, we use the GPT-2 tokenizer and apply right padding to obtain
fixed-length token sequences. Since the original GPT-2 tokenizer does not provide dedicated
padding or mask tokens, we add them when they are absent:
\[
    \langle \texttt{pad} \rangle, \qquad \langle \mask \rangle .
\]
Let \(\mathcal V\) denote the resulting tokenizer vocabulary, including the
original GPT-2 vocabulary and the newly added special tokens. Each text sample is
mapped to a sequence of token IDs in \(\{0,\ldots,|\mathcal V|-1\}\).

For the final inputs, we truncate or pad every sequence to a fixed length
\(L_{\max}\). Thus each text sample is represented as
\[
    x = (x_1,\ldots,x_{L_{\max}}),
    \qquad
    x_\ell \in \{0,\ldots,|\mathcal V|-1\}.
\]
Texts longer than \(L_{\max}\) tokens are truncated, and texts shorter than
\(L_{\max}\) tokens are padded on the right with \(\langle \texttt{pad} \rangle\).
Together with the token IDs, we store an attention mask
\[
    a = (a_1,\ldots,a_{L_{\max}}),
    \qquad
    a_\ell \in \{0,1\},
\]
where \(a_\ell=1\) indicates a valid token position and \(a_\ell=0\) indicates a padding position. Hence, after tokenization, each sample is represented by
\[
    (x,a)
    \in
    \{0,\ldots,|\mathcal V|-1\}^{L_{\max}}
    \times
    \{0,1\}^{L_{\max}}.
\]
The model is trained only on the tokenized normal examples from the training
split, while the full tokenized test split is used for evaluation. The same
tokenizer and maximum sequence length are used for both training and testing.










\subsection{Evaluation Criteria}\label{app:metric}

We evaluate anomaly detection performance using ROC-AUC and PR-AUC. 
For each test sample $\bx_i$, let $y_i\in\{0,1\}$ denote its ground-truth label, where $y_i=0$ indicates a normal sample and $y_i=1$ indicates an anomalous sample. Each method produces an anomaly score $S(\bx_i)$, where a larger score indicates stronger evidence of anomaly.

ROC-AUC measures the global ranking quality between anomalous and normal samples. Equivalently, it can be interpreted as the probability that a randomly chosen anomalous sample receives a higher score than a randomly chosen normal sample:
\[
\mathrm{AUC}_{\mathrm{ROC}}
=
\mathbb{P}\bigl(S(\bx^+) > S(\bx^-)\bigr),
\]
where $\bx^+$ denotes a randomly selected anomalous sample and $\bx^-$ denotes a randomly selected normal sample.

PR-AUC summarizes the precision--recall tradeoff obtained by thresholding the
anomaly scores. For a threshold $\gamma$, samples with $S(\bx_i)\ge \gamma$ are predicted
as anomalies, and
\[
\mathrm{Precision}(\gamma)=
\frac{\mathrm{TP}(\gamma)}{\mathrm{TP}(\gamma)+\mathrm{FP}(\gamma)},
\qquad
\mathrm{Recall}(\gamma)=
\frac{\mathrm{TP}(\gamma)}{\mathrm{TP}(\gamma)+\mathrm{FN}(\gamma)}.
\]
We report PR-AUC as average precision, which corresponds to a step-wise area under the precision--recall curve.

\subsection{\algname{} Hyperparameters}
\label{sec:hyperparameters}
For all variants of \algname{}, we use fixed hyperparameters across datasets rather than validation-based or per-dataset tuning. We adopt the standard linear absorbing schedule $\alpha_t=1-t$, so a probe level $\tau_\ell$ has mask probability $1-\alpha_{\tau_\ell}=\tau_\ell$. All reconstruction scores are averaged over the specified probe levels and independently generated masked probe views.

The parametric tabular model uses the uniform probe grid $\tau_\ell\in\{0.1,0.2,\ldots,0.9\}$ and $16$ masked probe views per probe level. Its reconstruction network uses embedding dimension $256$, hidden dimension $512$, three MLP layers, and dropout $0.1$. The model is trained for $20$ epochs with AdamW using learning rate $10^{-3}$ and weight decay $10^{-5}$.

The non-parametric tabular model uses the probe grid $\tau_\ell\in\{0.15,0.30,0.45,0.60\}$ and $8$ masked probe views per probe level. It has no trainable neural parameters and takes all normal training samples as the reference set. Masked-coordinate reconstruction probabilities are estimated using the visible-coordinate Hamming kernel $K_\lambda(\widetilde{\bx},\bx^{(n)})=\exp\{-\lambda d_{\rm vis}(\widetilde{\bx},\bx^{(n)})\}$ with bandwidth $\lambda=1$, where $d_{\rm vis}$ denotes the Hamming distance restricted to visible coordinates.

For text experiments, we use the GPT-2 tokenizer only, without using the pretrained GPT-2 language model or pretrained GPT-2 token embeddings. The text model uses the same uniform probe grid as the parametric tabular model, $\tau_\ell\in\{0.1,0.2,\ldots,0.9\}$, and $24$ masked probe views per probe level. The reconstruction model is an 8-layer Transformer encoder with hidden dimension $512$, $16$ attention heads, feed-forward dimension $2048$, and dropout $0.1$. It is trained for $20$ epochs with AdamW using learning rate $10^{-3}$ and weight decay $10^{-5}$.

\subsection{Baselines}
\label{sec:baselines}


We compare against classical anomaly detection baselines, DTE-based variants, and recent tabular or text anomaly detection methods. All baselines are evaluated under the same normal-only setting: each method is fitted only on normal training samples, and anomaly labels are used only for test-time evaluation.

\paragraph{Hyperparameter setup.}
We use a fixed default hyperparameter configuration across datasets and random seeds,
rather than performing per-dataset hyperparameter search. The main settings for all
tabular baselines are summarized in Table~\ref{tab:baseline_hyperparams}.

\paragraph{Tabular baselines.}
For categorical and mixed-type tabular datasets, categorical features are encoded
as discrete symbols. Numerical features are discretized by quantile binning using
only the normal training subset. By default, we use 10 quantile bins, treat
numerical features with at most 20 distinct values as discrete features, and use
an explicit unknown bucket for unseen categories. The fitted preprocessing
parameters are then reused for the test set.

Hamming-\(k\)NN, ECOD, and COPOD are applied to the ordinal/binned
representation. For methods whose behavior may be affected by artificial
ordering of category IDs, including Isolation Forest and neural tabular
baselines, we use the one-hot representation. DTE variants are also run on the
one-hot representation of the discretized features, so that they use the same
discretized information as the other baselines.

\paragraph{Text baselines.}
For text anomaly detection, we use datasets from NLP-ADBench~\cite{li2025nlp} and compare against the baseline results reported in that benchmark. NLP-ADBench evaluates 19 text anomaly detection baselines, including end-to-end methods
such as DATE~\cite{manolache-etal-2021-date} and
FATE~\cite{das2023fewshotanomalydetectiontext}, as well as embedding-based methods. For the embedding-based baselines, BERT embeddings~\cite{devlin-etal-2019-bert} and OpenAI text embeddings~\cite{openai_embeddings_2024} are used as fixed document
representations, followed by LOF~\cite{lof}, DeepSVDD~\cite{pmlr-v80-ruff18a}, ECOD~\cite{9737003}, Isolation Forest~\cite{iforest}, SO-GAAL~\cite{sogaal}, AE~\cite{aggarwal2017outlier}, VAE~\cite{kingma2013autoencoding}, and LUNAR~\cite{Goodge_Hooi_Ng_Ng_2022}. We evaluate our method on the same
datasets and report ROC-AUC and PR-AUC using the same evaluation convention.

\begin{table}[t]
\centering
\caption{Default hyperparameters for tabular baselines. The same settings are
used across datasets and seeds.}
\label{tab:baseline_hyperparams}
\begin{adjustbox}{max width=\textwidth}
\begin{tabular}{lll}
\toprule
Method & Input & Main settings \\
\midrule
COPOD & ordinal/binned
& PyOD default \\
DeepSVDD & one-hot
& 100 epochs, batch size 64, learning rate \(10^{-3}\), hidden dims \(100,50\) \\
DRL & one-hot
& 200 epochs, batch size 512, learning rate 0.05, LR decay 0.98 \\
DTE-Categorical / InvGamma / Gaussian & one-hot
& 100 epochs, batch size 128,
learning rate \(10^{-4}\), MLP hidden dims \(256,512,256\) \\
ECOD & ordinal/binned
& PyOD default \\
GOAD & one-hot
& 100 epochs, batch size 64, learning rate \(10^{-3}\), 256 transformations \\
Hamming-\(k\)NN & ordinal/binned
& \(k=25\), Hamming distance \\
ICL & one-hot
& 100 epochs, batch size 64, learning rate \(10^{-3}\), hidden dims \(100,50\) \\
IForest & one-hot
& 300 trees, \texttt{max\_samples=auto}, \texttt{contamination=auto} \\
MCM & one-hot
& 200 epochs, batch size 512, learning rate 0.05, LR decay 0.98 \\
\bottomrule
\end{tabular}
\end{adjustbox}
\end{table}

\subsection{Sensitivity Analysis}
\label{app:sensitivity_analysis}
\begin{figure}[t]
    \centering
    \begin{tabular}{cc}
    \includegraphics[width=0.47\linewidth]{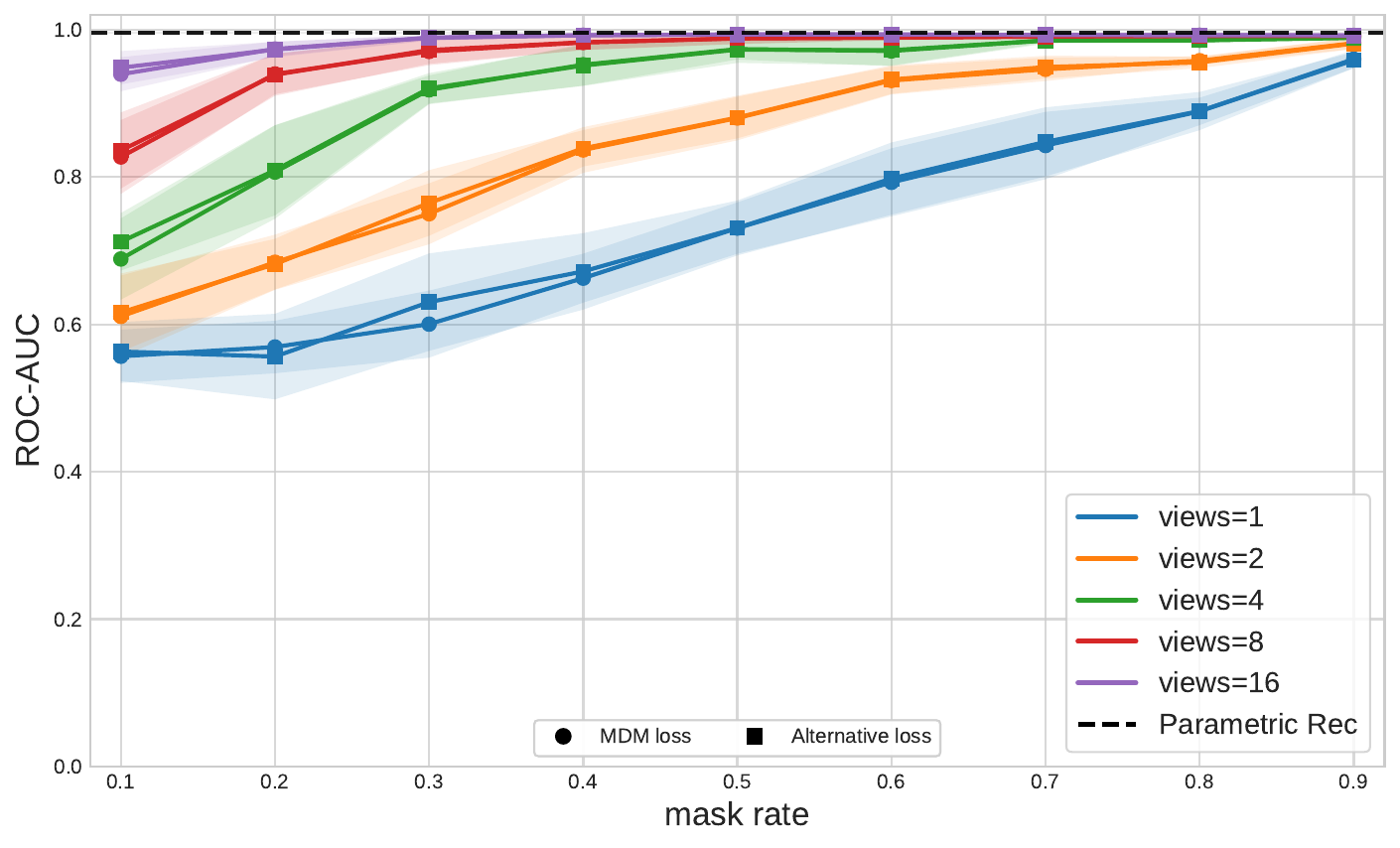} &
    \includegraphics[width=0.47\linewidth]{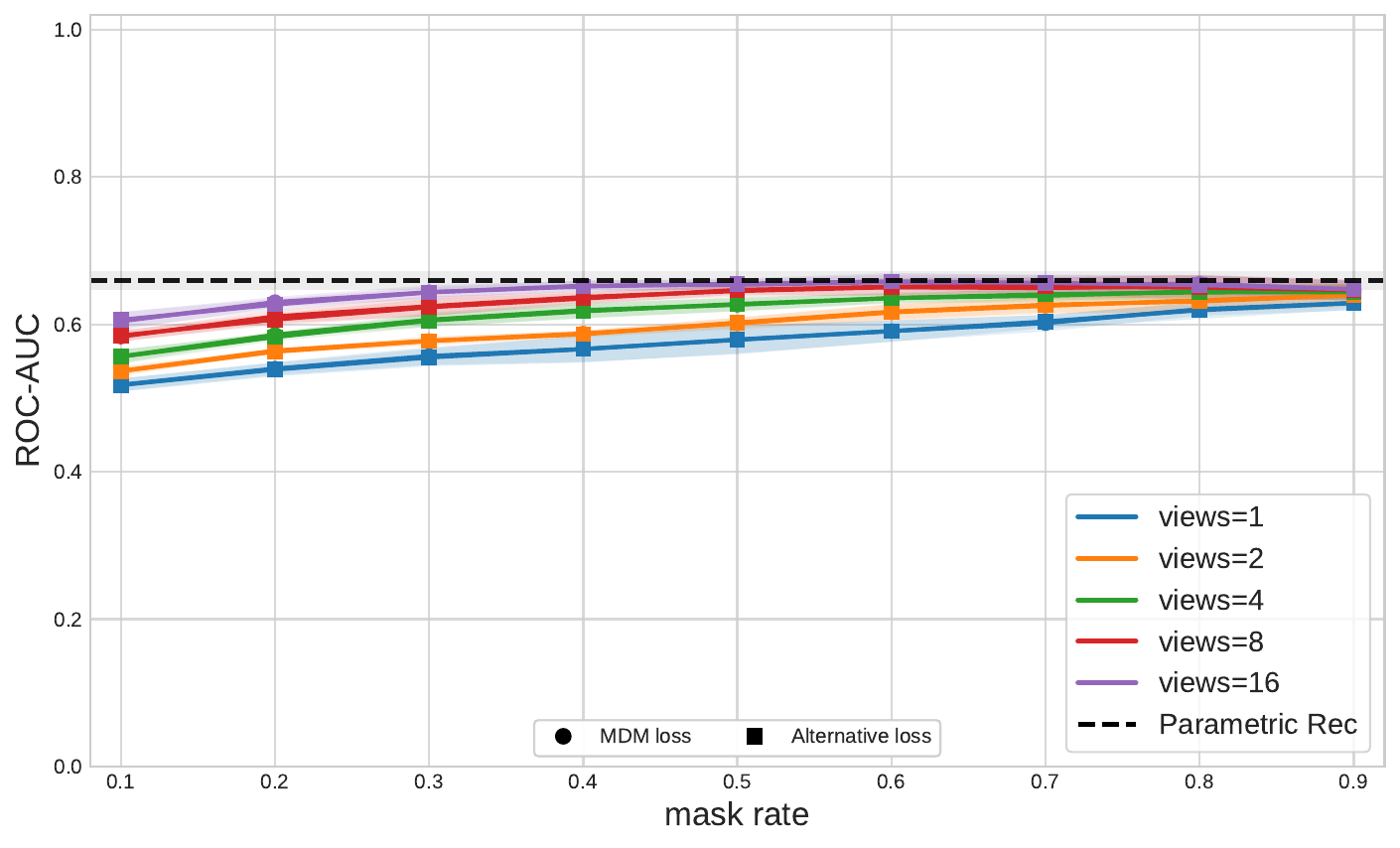}\\
    (a) U2R & (b) Bank
    \end{tabular}
    \caption{
    Additional sensitivity analysis of Parametric \algname{} to the probe mask rate $\tau$ on U2R and Bank. Together with the Vehicle Claims result in Figure~\ref{fig:probe_tau_sensitivity}, these results show that the optimal single probe level can vary across datasets, motivating the fixed multi-probe aggregation used in the main experiments.
    }
    \label{fig:probe_sensitivity_appendix}
\end{figure}

Figure~\ref{fig:probe_sensitivity_appendix} provides additional sensitivity results on U2R and Bank.
Compared with the Vehicle Claims result in Figure~\ref{fig:probe_tau_sensitivity}, these datasets exhibit different best-performing single probe levels.
This supports our choice of using a fixed multi-probe score that aggregates over a uniform probe grid, rather than tuning the probe level separately for each dataset.
In low-view regimes, the average masked reconstruction loss in
Eq.~\eqref{eq:para_rec_loss_view_level} is occasionally marginally better than the original masked diffusion loss on some datasets, possibly because it is more closely aligned with the normalized test-time reconstruction score.

\subsection{Runtime and Comparison}

We report runtime in Table~\ref{tab:runtime_vehicle_bank}. 
The comparison is conducted on two tabular datasets with different scales:
Bank contains 41,188 samples with 10 features, while Vehicle Claims contains
268,255 samples with 18 features. Runtime is decomposed into training or fitting
time, inference time, and total time.

The results show that Parametric \algname{} achieves a favorable balance between
training cost and inference efficiency. On the large Vehicle Claims dataset,
Parametric \algname{} takes 31.64 seconds for training and 16.28 seconds for inference,
leading to a total runtime of 47.92 seconds. This is substantially faster than
several deep baselines such as GOAD, ICL, and MCM, while also avoiding the high
test-time cost of Non-parametric \algname{}. On the Bank dataset, Parametric \algname{}
remains efficient, with a total runtime of 5.71 seconds.

The comparison also highlights the different computational profiles of the two
\algname\ variants. Non-parametric \algname{} does not require model training, but its
test-time cost can become large because each test sample is compared against the
normal training set through masked-view reconstruction. This effect is especially
clear on Vehicle Claims, where its inference time reaches 3005.69 seconds. In
contrast, Parametric \algname{} amortizes the reconstruction model during training and
then performs inference through a neural network, making it more
scalable on larger datasets. For smaller or medium-scale datasets such as Bank,
the non-parametric variant is still feasible, but the parametric version provides
a more consistent runtime profile across dataset sizes.

\begin{table}[t]
\centering
\caption{Runtime comparison on the Bank and Vehicle Claims datasets.}
\label{tab:runtime_vehicle_bank}
\begin{adjustbox}{max width=0.75\textwidth}
\begin{tabular}{lrrr rrr}
\toprule
& \multicolumn{3}{c}{Bank} 
& \multicolumn{3}{c}{Vehicle Claims} \\
\cmidrule(lr){2-4} \cmidrule(lr){5-7}
Method 
& Train / Fit & Inference & Total
& Train / Fit & Inference & Total \\
\midrule
Parametric \algname{} 
& 3.87 & 1.84 & 5.71
& 31.64 & 16.28 & 47.92 \\

Non-parametric \algname{} 
& 0.00 & 4.81 & 4.81
& 0.00 & 3005.69 & 3005.69 \\

COPOD 
& 0.01 & 0.02 & 0.03
& 0.13 & 0.22 & 0.35 \\

DeepSVDD 
& 18.73 & 0.03 & 18.75
& 130.88 & 0.20 & 131.09 \\

DRL 
& 20.24 & 0.02 & 20.25
& 132.68 & 0.12 & 132.80 \\

DTE-Categorical 
& 17.21 & 0.02 & 17.23
& 138.35 & 0.25 & 138.59 \\

DTE-Gaussian 
& 15.33 & 0.02 & 15.35
& 130.23 & 0.24 & 130.47 \\

DTE-InvGamma 
& 18.03 & 0.03 & 18.05
& 144.01 & 0.25 & 144.27 \\

ECOD 
& 0.12 & 0.02 & 0.14
& 0.24 & 0.25 & 0.49 \\

GOAD 
& 122.88 & 1.40 & 124.27
& 768.66 & 22.35 & 791.02 \\

Hamming-kNN 
& 0.00 & 3.17 & 3.17
& 0.00 & 65.66 & 65.66 \\

ICL 
& 61.86 & 0.11 & 61.97
& 817.37 & 2.02 & 819.39 \\

IForest 
& 0.25 & 0.09 & 0.35
& 0.35 & 1.35 & 1.70 \\

MCM 
& 49.69 & 0.04 & 49.72
& 678.01 & 0.63 & 678.64 \\
\bottomrule
\end{tabular}
\end{adjustbox}
\end{table}

\section{Proofs of Theoretical Guarantee}
\label{app:theory}


\begin{proof}[Proof of Theorem~\ref{thm:detection_performance}]
Recall the definitions for $S^\ast$ and $S_\rec^\ast$ in Eqs.~\eqref{def:S_ast} and \eqref{def:S_rec_ast}, respectively. With these, we also have
\begin{equation}
\label{eq:S_ast_diff}
\Delta_{\KL}(\bx) = S_\rec^\ast(\bx) - S^\ast(\bx) = \frac{1}{L}
    \sum_{\ell=1}^L
    \mathbb E_{\widetilde \bx \sim q_{\tau_\ell|0}(\cdot\mid \bx)}
    \bigg[
    \frac{
        \sum_{j\in M(\widetilde \bx)}
        \log \frac{q^j\bigl(x_j\mid \widetilde \bx,\tau_\ell\bigr)}{p_{\hat\theta}^j\bigl(x_j\mid \widetilde \bx,\tau_\ell\bigr)}
    }{
        |M(\widetilde \bx)|\vee 1
    }
    \bigg].
\end{equation}
When the log-likelihoods are bounded, we have that
\[ |S_\rec(\bx)| \le C,\quad |S_\rec^\ast(\bx)| \le C, \quad |S^\ast(\bx)| \le C,\quad |\Delta_{\KL}(\bx)| \le 2 C. \]

As follows, we provide the proofs for the Type-I error and the Type-II error, respectively. 

\paragraph{Type-I error} Note that 
\[S_{\mathrm{rec}}(\bx)= \underbrace{S_{\mathrm{rec}}(\bx) - S_\rec^\ast(\bx)}_{\text{stochastic test-time error from finite number of views}} + \underbrace{S_\rec^\ast(\bx) - S^\ast(\bx)}_{\text{training error from imperfect model}}  + S^\ast(\bx). \]
So, we have 
\begin{align*}
& \mathbb P_{\bx \sim q_0}(S_{\mathrm{rec}}(\bx) \ge \gamma)\\
&= \mathbb P_{\bx \sim q_0}( (S_{\mathrm{rec}}(\bx) - S_\rec^\ast(\bx)) + (S_\rec^\ast(\bx) - S^\ast(\bx) - \epsilon)  + (S^\ast(\bx) - \mu_0^\ast) \ge \gamma - \mu_0^\ast - \epsilon ) \\
& \stackrel{(i)}{\le}
\mathbb P_{\bx \sim q_0} \left( S_{\mathrm{rec}}(\bx) - S_\rec^\ast(\bx) \ge \frac{\gamma - \mu_0^\ast - \epsilon}{3} \right)+ 
\mathbb P_{\bx \sim q_0} \left( \Delta_{KL}(\bx) - \mathbb E_{\bx \sim q_0}[\Delta_{\KL}(\bx)] \ge \frac{\gamma - \mu_0^\ast - \epsilon}{3} \right) \\
&\qquad + \mathbb P_{\bx \sim q_0} \left( S^\ast(\bx) - \mu_0^\ast  \ge \frac{\gamma - \mu_0^\ast - \epsilon}{3} \right)\\
&\stackrel{(ii)}{\le} \exp\left(-\frac{2 K t^2}{C^2}\right) + \exp\left(-\frac{L t^2}{2 C^2}\right) + \exp\left(-\frac{2 Lt^2}{C^2}\right),
\end{align*}
where $t:=(\gamma - \mu_0^\ast - \epsilon)/3$, $(i)$ follows from the union bound that
\[
\mathbb{P}\left( \sum_{i=1}^3 U_i \ge u \right)
\le
\sum_{i=1}^3 \mathbb{P}(U_i \ge u/3),
\]
and $(ii)$ follows from Hoeffding's inequality.

\paragraph{Type-II error} The analysis is similar. When there is no estimation error, we have $S_\rec^\ast=S^\ast$, and thus 
\[ S_{\mathrm{rec}}(\bx)= S_{\mathrm{rec}}(\bx) - S_\rec^\ast(\bx) + S^\ast(\bx). \]
So, we have 
\begin{align*}
& \mathbb P_{q'}(S_{\mathrm{rec}}(\bx)< \gamma)\\
&= \mathbb P_{q'}( (S_{\mathrm{rec}}(\bx) - S_\rec^\ast(\bx)) + (S^\ast(\bx) - \mu_1^\ast) < \gamma - \mu_1^\ast ) \\
& \stackrel{(iii)}{\le}
\mathbb P_{q'}\left( S_{\mathrm{rec}}(\bx) - S_\rec^\ast(\bx)< \frac{\gamma - \mu_1^\ast }{2} \right)+ 
\mathbb P_{q'} \left( S^\ast(\bx) - \mu_1^\ast  < \frac{\gamma - \mu_1^\ast }{2}  \right)\\
&\stackrel{(iv)}{\leq} \exp\left(-\frac{K (\mu_1^\ast - \gamma)^2}{2 C^2}\right) + \exp\left(-\frac{L (\mu_1^\ast - \gamma)^2}{2 C^2}\right),
\end{align*}
where $(iii)$ holds again by union bound:
\[
\mathbb{P}\left( U_1 + U_2 \le -u \right)
\le
\sum_{i=1}^2 \mathbb{P}(U_i \le -u/2),
\]
and $(iv)$ holds again by Hoeffding's inequality. \qedhere

\end{proof}

\begin{table}[p]
\centering
\scriptsize
\begin{minipage}[c]{0.48\textwidth}
\centering
\rotatebox{90}{%
\begin{minipage}{0.9\textheight}
\centering
\captionof{table}{ROC-AUC on tabular datasets over five random seeds.}
\label{tab:rocauc_seed_avg}
\begin{adjustbox}{max width=\linewidth}
\begin{tabular}{lcccccccccccccc}
\toprule
Method & ad & aid362 & apas & bank & census & chess & cmc & covertype & probe & r10 & solar & u2r & vehicle\_claims & w7a \\
\midrule
Parametric \algname{} & 78.86 (2.34) & 67.93 (4.78) & 83.15 (3.36) & 65.90 (0.91) & 66.62 (0.52) & 73.55 (7.74) & 59.30 (8.97) & 91.92 (0.68) & 98.42 (0.09) & 99.14 (0.14) & 86.00 (4.61) & 99.55 (0.10) & \textbf{96.38 (0.09)} & 62.53 (1.93) \\
Non-parametric \algname{}  & \textbf{93.43 (1.72)} & 64.69 (4.70) & 80.01 (1.85) & 63.68 (0.89) & 67.01 (0.15) & 55.96 (11.17) & 56.85 (11.64) & 95.47 (0.21) & 98.38 (0.09) & 99.09 (0.12) & 85.85 (3.48) & 99.11 (0.12) & 83.14 (0.15) & 65.05 (1.44) \\
COPOD & 73.85 (2.08) & 65.12 (4.77) & 64.15 (3.96) & 59.31 (2.19) & 64.45 (1.25) & 70.76 (27.06) & 51.07 (10.94) & 97.72 (0.08) & 98.33 (0.08) & 98.83 (0.13) & 87.04 (5.04) & 97.74 (0.67) & 67.18 (0.27) & 58.68 (1.49) \\
DeepSVDD & 74.09 (12.18) & 59.93 (12.83) & 57.45 (17.30) & 52.05 (2.73) & 58.57 (7.34) & 45.68 (10.43) & 62.17 (5.69) & 47.84 (29.21) & 83.32 (4.83) & 96.88 (3.31) & 81.76 (3.32) & 76.77 (24.13) & 77.41 (1.46) & 81.23 (3.32) \\
DTE-Categorical & 85.18 (2.56) & 66.53 (3.26) & 86.37 (2.16) & \textbf{68.05 (1.06)} & 67.17 (0.75) & 51.97 (16.57) & \textbf{63.26 (6.77)} & 94.21 (1.06) & 97.88 (0.26) & 98.91 (0.16) & 78.06 (9.51) & 99.01 (0.15) & 65.82 (0.30) & 81.56 (1.10) \\
DTE-InvGamma & 79.30 (6.47) & 58.86 (11.71) & 75.50 (6.23) & 58.19 (5.85) & 61.98 (4.16) & 57.41 (7.37) & 52.98 (13.21) & 88.38 (8.25) & 91.65 (5.73) & 90.30 (4.41) & 72.58 (7.86) & 85.17 (29.79) & 87.70 (0.84) & 81.39 (1.54) \\
DTE-Gaussian & 91.27 (1.19) & 66.57 (5.09) & 72.40 (2.83) & 62.52 (0.90) & 61.34 (3.37) & 57.71 (12.17) & 56.48 (8.72) & 64.68 (15.46) & 98.45 (0.10) & \textbf{99.32 (0.11)} & 82.84 (4.87) & \textbf{99.61 (0.07)} & 84.65 (0.54) & \textbf{84.13 (2.55)} \\
DRL & 79.95 (1.49) & 65.59 (6.76) & 73.87 (6.00) & 58.72 (2.08) & 58.62 (8.07) & 45.64 (13.10) & 61.60 (8.43) & 44.53 (16.41) & \textbf{98.88 (0.24)} & 99.09 (0.48) & 82.13 (5.89) & 97.04 (3.18) & 68.31 (5.31) & 64.01 (12.23) \\
ECOD & 73.92 (2.08) & 65.59 (4.63) & 67.21 (3.51) & 60.90 (1.67) & 62.89 (2.28) & 62.21 (25.77) & 54.13 (9.53) & \textbf{97.82 (0.07)} & 98.35 (0.10) & 98.83 (0.13) & 85.78 (4.44) & 97.53 (1.38) & 66.86 (0.51) & 59.13 (1.48) \\
GOAD & 74.72 (2.42) & 65.00 (4.48) & 65.10 (4.91) & 57.06 (1.25) & 60.36 (1.24) & 62.28 (12.66) & 53.27 (9.48) & 95.43 (0.82) & 98.39 (0.10) & 98.72 (0.14) & \textbf{87.06 (3.97)} & 98.96 (0.12) & 58.04 (0.96) & 50.03 (1.84) \\
Hamming-kNN & 80.97 (2.02) & \textbf{71.69 (3.49)} & \textbf{86.89 (1.90)} & 64.50 (0.77) & 66.40 (0.27) & \textbf{93.22 (1.87)} & 60.59 (9.43) & 50.00 (0.00) & 89.17 (0.77) & 98.90 (0.12) & 83.67 (5.76) & 87.87 (1.53) & 72.98 (0.19) & 52.36 (1.35) \\
ICL & 79.69 (4.19) & 60.49 (3.52) & 44.57 (18.25) & 60.81 (2.53) & 65.53 (1.94) & 57.56 (7.41) & 57.17 (18.61) & 36.11 (21.32) & 23.85 (7.33) & 98.70 (0.27) & 82.33 (9.17) & 28.50 (31.30) & 81.43 (2.65) & 76.58 (2.93) \\
IForest & 42.17 (1.58) & 66.49 (4.75) & 52.75 (1.09) & 57.49 (0.94) & 59.32 (1.80) & 68.19 (14.57) & 56.49 (9.64) & 93.83 (2.45) & 96.67 (0.26) & 96.90 (0.51) & 83.16 (5.45) & 96.69 (0.85) & 46.15 (1.71) & 41.26 (1.69) \\
MCM & 75.47 (5.24) & 63.45 (5.47) & 67.51 (3.48) & 60.73 (1.94) & \textbf{68.18 (0.40)} & 56.73 (22.41) & 55.83 (12.21) & 91.43 (5.43) & 95.78 (5.30) & 98.61 (0.20) & 64.54 (19.06) & 95.65 (5.55) & 69.75 (2.51) & 54.05 (5.72) \\
\bottomrule
\end{tabular}
\end{adjustbox}
\end{minipage}%
}
\end{minipage}
\hfill
\begin{minipage}[c]{0.48\textwidth}
\centering
\rotatebox{90}{%
\begin{minipage}{0.9\textheight}
\centering
\captionof{table}{PR-AUC on tabular datasets over five random seeds.}
\label{tab:prauc_seed_avg}
\begin{adjustbox}{max width=\linewidth}
\begin{tabular}{lcccccccccccccc}
\toprule
Method & ad & aid362 & apas & bank & census & chess & cmc & covertype & probe & r10 & solar & u2r & vehicle\_claims & w7a \\
\midrule
Parametric \algname{}  & 64.40 (3.96) & 3.76 (0.94) & 4.91 (1.17) & 23.90 (1.19) & 9.06 (0.09) & 0.69 (0.48) & 5.89 (3.37) & 5.50 (0.83) & 91.59 (0.68) & \textbf{70.34 (3.75)} & 20.29 (6.55) & 46.37 (6.45) & \textbf{93.04 (0.06)} & 13.71 (2.47) \\
Non-parametric \algname{}  & \textbf{86.82 (3.30)} & 4.02 (0.88) & 5.37 (1.14) & 21.39 (0.83) & 9.05 (0.11) & 0.29 (0.30) & 5.45 (4.98) & 9.69 (0.40) & 89.88 (0.34) & 69.03 (3.78) & 17.48 (4.24) & 34.60 (4.06) & 74.39 (0.28) & 15.54 (2.56) \\
COPOD & 56.18 (4.18) & 2.83 (0.37) & 2.46 (0.44) & 17.42 (2.41) & 8.34 (0.32) & 0.68 (0.82) & 3.95 (2.19) & 12.84 (0.47) & 82.92 (0.67) & 60.20 (4.51) & 21.89 (6.23) & 35.63 (12.34) & 38.89 (0.42) & 4.66 (0.40) \\
DeepSVDD & 52.23 (17.63) & 2.61 (1.23) & 2.63 (1.67) & 13.84 (0.83) & 8.74 (2.10) & 0.10 (0.03) & 5.11 (1.89) & 1.57 (1.12) & 74.61 (4.29) & 68.71 (5.30) & 22.14 (5.65) & 33.54 (18.24) & 71.22 (1.15) & \textbf{55.84 (7.10)} \\
DTE-Categorical & 59.87 (7.51) & 3.24 (0.28) & 5.11 (1.09) & \textbf{27.71 (1.74)} & 9.12 (0.21) & 0.15 (0.10) & 3.83 (0.61) & 4.98 (0.25) & 73.54 (2.60) & 52.03 (4.20) & \textbf{22.90 (10.74)} & 22.83 (2.66) & 34.93 (0.27) & 40.88 (4.20) \\
DTE-InvGamma & 55.73 (15.85) & 4.50 (2.79) & 5.56 (2.33) & 18.83 (3.44) & 8.11 (0.74) & 0.14 (0.04) & 3.19 (0.67) & 6.04 (2.39) & 79.05 (6.78) & 30.92 (9.94) & 19.15 (8.40) & 41.56 (29.44) & 82.69 (1.46) & 52.83 (3.56) \\
DTE-Gaussian & 81.42 (2.97) & \textbf{5.07 (2.18)} & 4.38 (1.07) & 19.81 (1.39) & 8.30 (0.42) & 0.18 (0.07) & 3.68 (1.61) & 2.51 (0.90) & 91.86 (0.52) & 66.04 (5.54) & 21.82 (8.78) & \textbf{71.36 (3.45)} & 72.39 (0.93) & 45.98 (3.46) \\
DRL & 53.25 (5.60) & 3.58 (1.22) & 3.24 (1.23) & 16.11 (1.73) & 7.23 (1.41) & 0.10 (0.05) & 6.42 (3.07) & 1.36 (1.92) & \textbf{92.07 (0.57)} & 66.69 (12.28) & 21.93 (7.59) & 57.29 (15.54) & 53.97 (8.83) & 24.22 (14.89) \\
ECOD & 56.29 (4.17) & 2.85 (0.37) & 2.63 (0.42) & 19.16 (2.18) & 8.01 (0.54) & 0.28 (0.21) & 3.82 (1.89) & \textbf{13.00 (0.47)} & 83.03 (1.61) & 60.22 (4.51) & 21.85 (5.43) & 33.75 (15.29) & 38.53 (0.61) & 4.77 (0.41) \\
GOAD & 54.98 (4.92) & 2.72 (0.35) & 2.62 (0.68) & 15.40 (1.06) & 7.27 (0.23) & 0.30 (0.28) & \textbf{6.98 (8.48)} & 10.21 (2.13) & 90.09 (1.05) & 56.83 (3.89) & 21.23 (6.15) & 47.75 (12.90) & 26.84 (0.68) & 2.96 (0.17) \\
Hamming-kNN & 63.64 (3.75) & 3.48 (0.57) & \textbf{6.91 (1.09)} & 21.16 (0.88) & 8.76 (0.10) & \textbf{1.86 (1.34)} & 5.99 (5.52) & 0.47 (0.00) & 78.90 (1.40) & 59.61 (4.34) & 17.87 (5.56) & 60.29 (4.39) & 46.22 (0.32) & 3.07 (0.11) \\
ICL & 58.13 (6.48) & 3.32 (0.80) & 1.33 (0.50) & 20.19 (3.51) & 8.54 (0.60) & 0.47 (0.58) & 4.26 (3.15) & 1.05 (0.65) & 21.48 (7.54) & 61.78 (3.56) & 19.85 (8.81) & 22.91 (30.33) & 74.25 (2.89) & 33.82 (3.66) \\
IForest & 11.16 (0.27) & 3.02 (0.56) & 1.58 (0.12) & 15.23 (0.49) & 7.03 (0.33) & 0.50 (0.67) & 4.75 (3.37) & 5.93 (2.15) & 67.72 (2.57) & 39.91 (5.05) & 18.50 (5.90) & 19.14 (1.50) & 19.47 (0.75) & 2.32 (0.08) \\
MCM & 47.72 (9.38) & 2.63 (0.61) & 2.31 (0.53) & 19.20 (2.98) & \textbf{10.14 (0.24)} & 1.80 (3.27) & 3.85 (1.48) & 6.83 (2.34) & 85.97 (4.36) & 56.23 (7.23) & 14.25 (11.54) & 42.34 (22.82) & 41.25 (3.46) & 7.16 (2.24) \\
\bottomrule
\end{tabular}
\end{adjustbox}
\end{minipage}%
}
\end{minipage}

\end{table}

\end{document}